\documentclass{article}

\usepackage{microtype}
\usepackage{graphicx}
\usepackage{subfigure}
\usepackage{booktabs}
\usepackage{hyperref}
\usepackage{url}
\usepackage{amsthm}
\usepackage{mathtools}
\usepackage{amsfonts}
\usepackage{amsmath}

\newtheorem{definition}{Definition}
\newtheorem{proposition}{Proposition}
\newtheorem{property}{Property}

\newcommand{\E}{\mathbb{E}}

\newcommand{\gN}{\mathcal{N}}
\newcommand{\calO}{\mathcal{O}}
\newcommand{\R}{\mathbb{R}}

\DeclareMathOperator*{\argmax}{arg\,max}

\def\[#1\]{\begin{align}#1\end{align}}

\usepackage[accepted]{icml2019}

\icmltitlerunning{Set Transformer}

\begin{document}

\twocolumn[
\icmltitle{Set Transformer: A Framework for Attention-based \\Permutation-Invariant Neural Networks}

\icmlsetsymbol{equal}{*}

\begin{icmlauthorlist}
\icmlauthor{Juho Lee}{oxford1,aitrics}
\icmlauthor{Yoonho Lee}{kakao}
\icmlauthor{Jungtaek Kim}{postech}
\icmlauthor{Adam R. Kosiorek}{oxford1,oxford2}
\icmlauthor{Seungjin Choi}{postech}
\icmlauthor{Yee Whye Teh}{oxford1}
\end{icmlauthorlist}

\icmlaffiliation{oxford1}{Department of Statistics, University of Oxford, United Kingdom}
\icmlaffiliation{oxford2}{Oxford Robotics Institute, University of Oxford, United Kingdom}
%\icmlaffiliation{oxford2}{Applied Artificial Intelligence Lab, Oxford Robotics Institute, University of Oxford}
\icmlaffiliation{postech}{Department of Computer Science and Engineering, POSTECH, Republic of Korea}
\icmlaffiliation{aitrics}{AITRICS, Republic of Korea}
\icmlaffiliation{kakao}{Kakao Corporation, Republic of Korea}

\icmlcorrespondingauthor{Juho Lee}{juho.lee@stats.ox.ac.uk}

\icmlkeywords{Set transformer, Set-taking network, Attention-based network, Permutation-invariant network}

\vskip 0.3in
]



\printAffiliationsAndNotice{}  % leave blank if no need to mention equal contribution

\begin{abstract}
Many machine learning tasks such as multiple instance learning, 3D shape recognition and few-shot image classification are defined on sets of instances.
Since solutions to such problems do not depend on the order of elements of the set, models used to address them should be \textit{permutation invariant}.
We present an attention-based neural network module, the \textit{Set Transformer}, specifically designed to model interactions among elements in the input set.
The model consists of an encoder and a decoder, both of which rely on attention mechanisms.
In an effort to reduce computational complexity, we introduce an attention scheme inspired by inducing point methods from sparse Gaussian process literature.
It reduces computation time of self-attention from quadratic to linear in the number of elements in the set.
We show that our model is theoretically attractive and we evaluate it on a range of tasks, demonstrating increased performance compared to recent methods for set-structured data.
\end{abstract}

\section{Introduction}
\label{sec:introduction}

Learning representations has proven to be an essential problem for deep learning and its many success stories.
The majority of problems tackled by deep learning are \textit{instance-based} and take the form of mapping a fixed-dimensional input tensor to its corresponding target value~\citep{krizhevsky2012imagenet, graves2013speech}.
\if{0}
There are also \textit{sequential} problems, often treated by the means of recurrent neural networks (RNNs), which can process inputs of variable length.
Learning long-range dependencies remains an issue, however, and often requires various engineering tricks~\citep{Le2015, Neil2016}.
\fi

For some applications, we are required to process \emph{set-structured data}.
Multiple instance learning~\citep{Dietterich1997, Maron1998} is an example of such a \textit{set-input} problem, where a set of instances is given as an input and the corresponding target is a label for the entire set.
Other problems such as 3D shape recognition~\citep{Wu2015, BaoguangShi2015, Su2015, Charles2017}, sequence ordering~\citep{Vinyals2016a}, and various set operations~\citep{muandet2012learning, oliva2013distribution, Edwards2017, Zaheer2017} can also be viewed as the set-input problems.
Moreover, many meta-learning~\citep{ThrunS98book, SchmidhuberJ87phd} problems which learn using different, 
but related tasks may also be treated as set-input tasks where an input set corresponds to the training dataset of a single task.
For example, few-shot image classification~\citep{finn2017model, Snell2017, Lee2018} operates by building a classifier using a support set of images, which is evaluated with query images.

A model for \textit{set-input} problems should satisfy two critical requirements.
First, it should be \textit{permutation invariant} --- the output of the model should not change under any permutation of the elements in the input set.
Second, such a model should be able to process input sets of any size.
While these requirements stem from the definition of a set, they are not easily satisfied in neural-network-based models: classical feed-forward neural networks violate both requirements, and RNNs are sensitive to input order.

Recently, \citet{Edwards2017} and \citet{Zaheer2017} propose neural network architectures which meet both criteria, which we call \emph{set pooling} methods.
In this model, each element in a set is first independently fed into a feed-forward neural network that takes fixed-size inputs.
Resulting feature-space embeddings are then aggregated using a \emph{pooling} operation ($\operatorname{mean}$, $\operatorname{sum}$, $\operatorname{max}$ or similar).
The final output is obtained by further non-linear processing of the aggregated embedding.
This remarkably simple architecture satisfies both aforementioned requirements, and more importantly, is proven to be a universal approximator for any set function~\citep{Zaheer2017}.
Thanks to this property, it is possible to learn a complex mapping between input sets and their target outputs in a black-box fashion, much like with feed-forward or recurrent neural networks.

Even though this set pooling approach is theoretically attractive, it remains unclear whether we can approximate complex mappings well using only instance-based feature extractors and simple pooling operations.
Since every element in a set is processed independently in a set pooling operation, some information regarding interactions between elements has to be necessarily discarded.
This can make some problems unnecessarily difficult to solve.

Consider the problem of \emph{amortized clustering}, where we would like to learn a parametric mapping from an input set of points to the centers of clusters of points inside the set.
Even for a toy dataset in 2D space, this is not an easy problem.
The main difficulty is that the parametric mapping must assign each point to its corresponding cluster while modelling the explaining away pattern such that the resulting clusters do not attempt to explain overlapping subsets of the input set.
Due to this innate difficulty, clustering is typically solved via iterative algorithms that refine randomly initialized clusters until convergence.
Even though a neural network with a set poling operation can approximate such an amortized mapping by learning to quantize space, a crucial shortcoming is that this quantization cannot depend on the contents of the set.
This limits the quality of the solution and also may make optimization of such a model more difficult;
we show empirically in Section~\ref{sec:experiments} that such pooling architectures suffer from under-fitting.

%\ak{I don't agree with the following. The instance-based mapping can learn a partitioning of space, think e.g. a voxel-based representation of a 3D space; the output embedding can place every point in a corresponding partition. Summing/pooling these points can recover the spatial structure of the input set. The fact that we don't really learn that might be an optimization issue}
%\ljh{I would not say the instance-based mapping `cannot' learn a partitioning of space, but it would be much 'harder' to learn, just as the neural nets without attention in theory can learn any function that is learnable with attention. Is their any moderate way of saying this?}
%\ak{a neural net yes, but not necessarily a neural net with a set-pooling op}
%\textcolor{gray}{
%Additionally, the simple pooling operation may cause considerable information loss.
%For instance, think of a meta-clustering problem, where the goal is to learn a mapping from sets of data points to cluster centers.
%Pairwise or higher-order interactions among data points should be crucial for learning a mapping from sets of data to cluster centers, but set pooling does not easily encode these.}

In this paper, we propose a novel set-input deep neural network architecture called the \emph{Set Transformer}, (\textit{cf.} \emph{Transformer},~\cite{Vaswani2017}).
The novelty of the Set Transformer is in three important design choices:
\begin{enumerate}
%1) We use a self-attention mechanism based on the Transformer to process every element in an input set, which allows our approach to naturally encode pairwise- or higher-order interactions between elements in the set.
\item We use a self-attention mechanism to process every element in an input set, which allows our approach to naturally encode pairwise- or higher-order interactions between elements in the set.
\item We propose a method to reduce the $\calO(n^2)$ computation time of full self-attention (e.g. the Transformer) to $\calO(nm)$ where $m$ is a fixed hyperparameter, allowing our method to scale to large input sets.
\item We use a self-attention mechanism to aggregate features, which is especially beneficial when the problem requires multiple outputs which depend on each other, such as the problem of meta-clustering, where the meaning of each cluster center heavily depends its location relative to the other clusters.
\end{enumerate}
We apply the Set Transformer to several set-input problems and empirically demonstrate the  importance and effectiveness of these design choices, and show that we can achieve the state-of-the-art performances for the most of the tasks.

\iffalse
This paper is organized as follows. In Section~\ref{sec:background}, we briefly review the concept of set functions, existing architectures, and the self-attention mechanism.
In Section~\ref{sec:set_transformer}, we introduce Set Transformers, our novel neural network architecture for set functions.
We discuss related works in Section~\ref{sec:related_works} and present various experiments that demonstrate the benefits of the Set Transformer in Section~\ref{sec:experiments}.
We conclude the paper in Section~\ref{sec:conclusion}.
\fi

\section{Background}
\label{sec:background}

\iffalse
\subsection{Permutation invariant models for set data}
\label{subsec:permutation}

Problems that involve a set of objects have the property that the target value is the same regardless of the order of these objects.
We say that such functions are $\textit{permutation invariant}$:

\begin{property}
Any set function $f:2^X \rightarrow Y$ is \newterm{permutation invariant}, meaning that for any sequence $(x_1, \cdots, x_n) \in X^n$ and any permutation $\pi$ of indices $\{1, \dots, n\}$,
\[
f(\{ x_1, \dots, x_n \}) = f(\{ x_{\pi(1)}, \dots, x_{\pi(n)}\}).
\]
\end{property}
\begin{proof}
Obvious since $\{ x_1, \dots, x_n \}$ and $\{x_{\pi(1)}, \dots, x_{\pi(n)} \}$ are the same object.
\end{proof}

Just as convolutional networks benefit from their built-in inductive bias about the translation invariance of image features, only being able to parameterize permutation invariant functions is a desirable property in the case of set inputs.
\fi

\subsection{Pooling Architecture for Sets}
\label{subsec:pooling}

Problems involving a set of objects have the \textit{permutation invariance} property: the target value for a given set is the same regardless of the order of objects in the set.
A simple example of a permutation invariant model is a network that performs pooling over embeddings extracted from the elements of a set.
More formally,
\[
\label{eq:deepset}
\mathrm{net}(\{x_1, \ldots, x_n \}) = \rho( \mathrm{pool}( \{ \phi(x_1), \ldots, \phi(x_n) \})).
\]
\citet{Zaheer2017} have proven that all permutation invariant functions can be represented as \eqref{eq:deepset} when $\mathrm{pool}$ is the $\operatorname{sum}$ operator and $\rho, \phi$ any continuous functions, thus justifying the use of this architecture for set-input problems.

Note that we can deconstruct \eqref{eq:deepset} into two parts:
an $\textit{encoder}$ ($\phi$) which independently acts on each element of a set of $n$ items, and a $\textit{decoder}$ ($\rho( \mathrm{pool}(\cdot))$) which aggregates these encoded features and produces our desired output.
Most network architectures for set-structured data follow this encoder-decoder structure.

\citet{Zaheer2017} additionally observed that the model remains permutation invariant even if the encoder is a stack of permutation-equivariant layers:

\begin{definition}
Let $S_n$ be the set of all permutations of indices $\{1, \ldots, n\}$.
A function $f: X^n \rightarrow Y^n$ is permutation equivariant iff for any permutation $\pi \in S_n$, $f(\pi x) = \pi f(x)$.
\end{definition}

An example of a permutation-equivariant layer is
\[
f_i(x; \{x_1, \ldots, x_n\}) = \sigma_i(\lambda x + \gamma \mathrm{pool}(\{ x_1, \ldots, x_n \}))
\]
where $\mathrm{pool}$ is the pooling operation, $\lambda, \gamma$ are learnable scalar variables, and $\sigma(\cdot)$ is a nonlinear activation function.

\begin{figure*}[h]
	\centering
	\subfigure[Our model]
	{
		\includegraphics[width=0.26\textwidth]{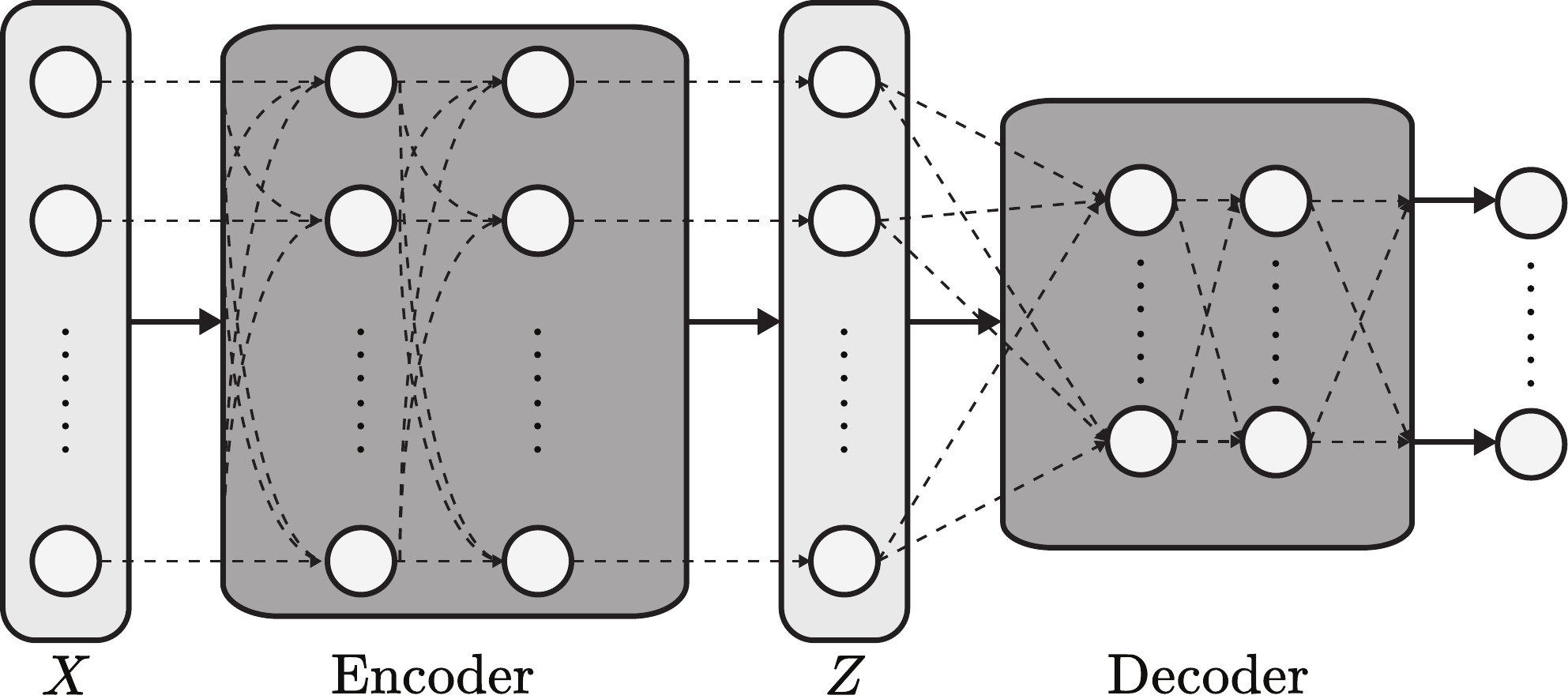}
	}
	\subfigure[MAB]
	{
		\includegraphics[width=0.22\textwidth]{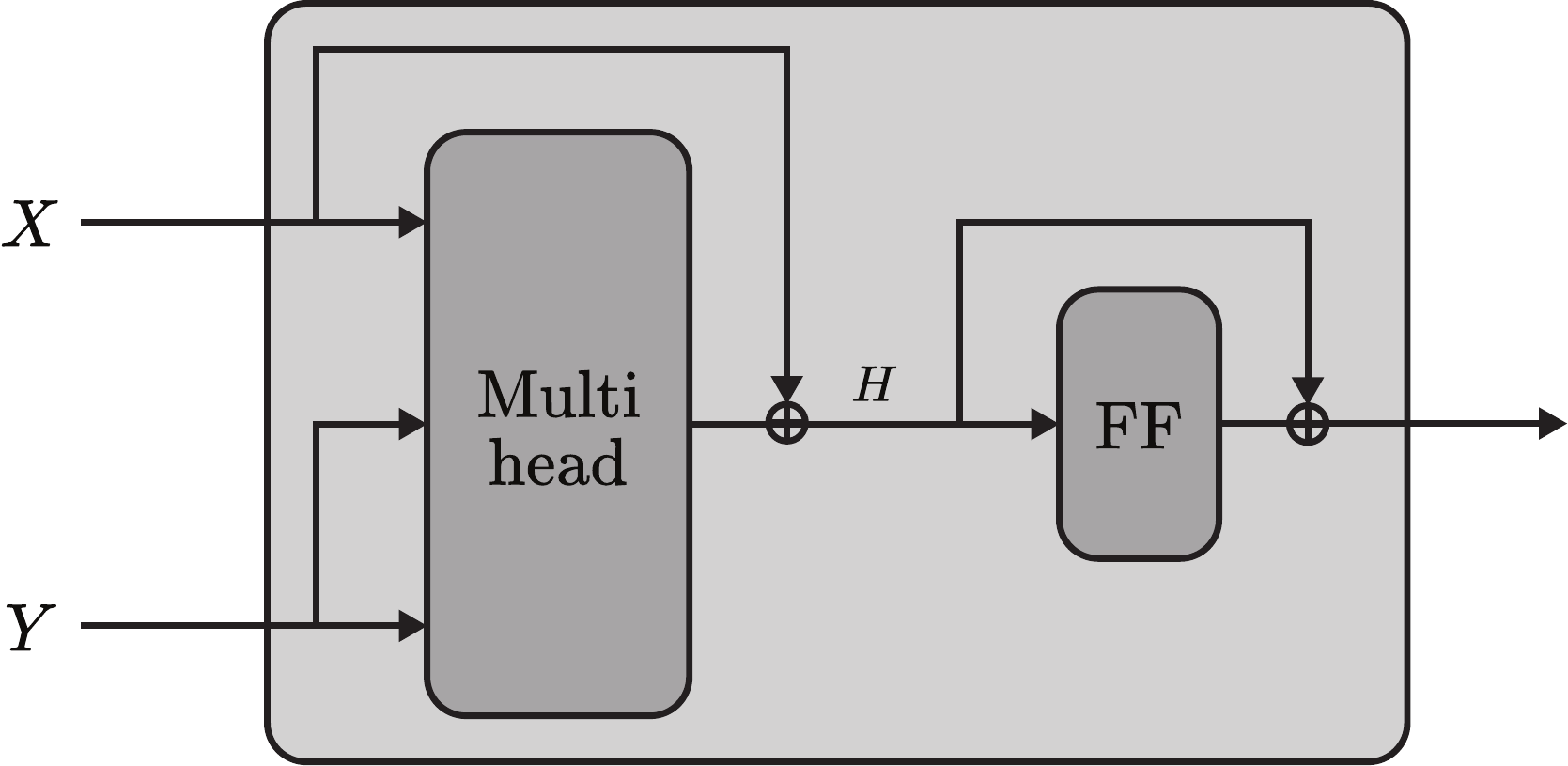}
	}
	\subfigure[SAB]
	{
		\includegraphics[width=0.22\textwidth]{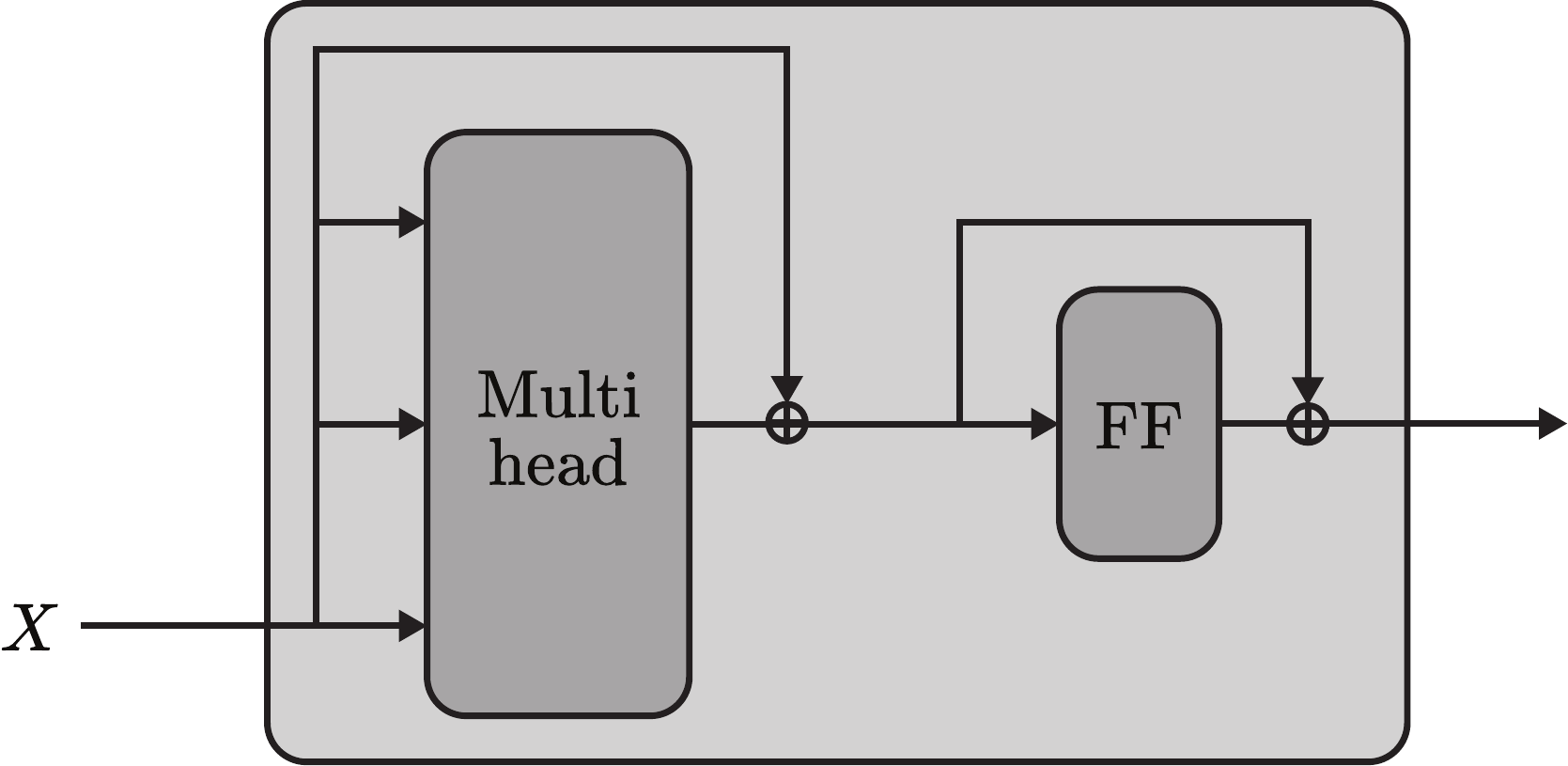}
	}
	\subfigure[ISAB]
	{
		\includegraphics[width=0.22\textwidth]{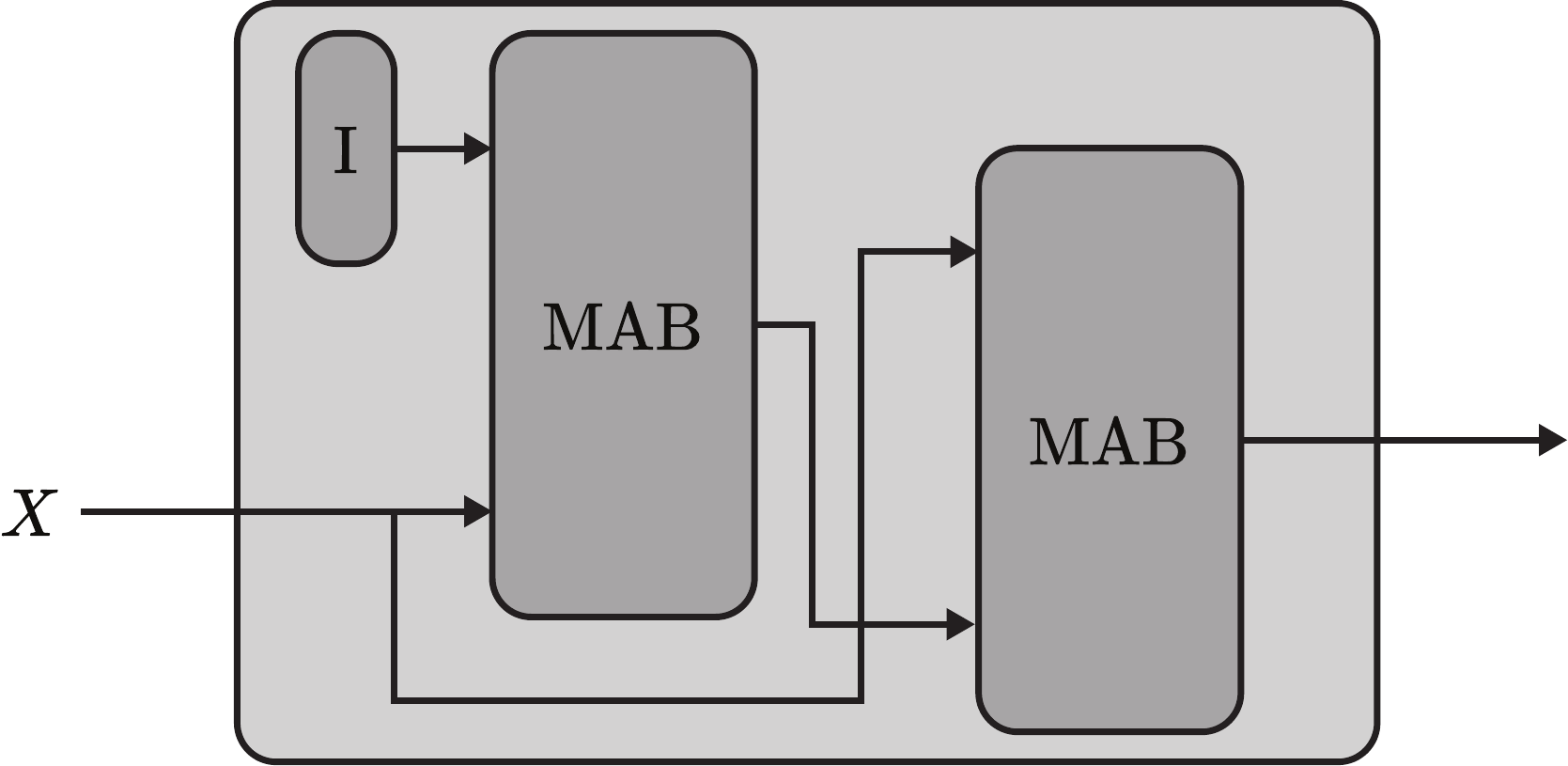}
	}
	\caption{Diagrams of our attention-based set operations.}
\end{figure*}
\subsection{Attention}
\label{subsec:attention}
Assume we have $n$ query vectors (corresponding to a set with $n$ elements) each with dimension $d_q$:~$Q \in \R^{n \times d_q}$.
An attention function $\mathrm{Att}(Q, K, V)$ is a function that maps queries $Q$ to outputs using $n_v$ key-value pairs $K \in \R^{n_v \times d_q}, V \in \R^{n_v \times d_v}$.
\[
\mathrm{Att}(Q, K, V;\omega) = \omega \left( Q K^\top \right) V. \label{eq:att}
\]
The pairwise dot product $QK^\top \in \R^{n \times n_v}$ measures how similar each pair of query and key vectors is, with weights computed with an activation function $\omega$. The output $\omega(QK^\top)V$ is a weighted sum of $V$ where a value gets more weight if its corresponding key has larger dot product with the query.

\textit{Multi-head attention}, originally introduced in \citet{Vaswani2017}, is an extension of the previous attention scheme.
Instead of computing a single attention function, this method first projects $Q, K, V$ onto $h$ different $d^M_q, d^M_q, d^M_v$-dimensional vectors, respectively.
An attention function ($\mathrm{Att}(\cdot; \omega_j)$) is applied to each of these $h$ projections.
The output is a linear transformation of the concatenation of all attention outputs:
\[
\mathrm{Multihead}&(Q, K, V ; \lambda,\omega) = \mathrm{concat}(O_1, \cdots, O_h) W^O, \label{eq:multihead1}\\
\mathrm{where}& \ \ O_j = \mathrm{Att}(Q W_j^Q, KW^K_j, V W^V_j;\omega_j) \label{eq:multihead2}
\]
Note that $\mathrm{Multihead}(\cdot, \cdot, \cdot;\lambda)$ has learnable parameters
$\lambda = \{W_j^Q, W_j^K, W_j^V\}_{j=1}^h$, where $W_j^Q, W_j^K \in \R^{d_q \times d_q^M}$, $W_j^V \in \R^{d_v \times d_v^M}$, $W^O \in \R^{hd_v^M \times d}$.
A typical choice for the dimension hyperparameters is $d_q^M=d_q / h$, $d_v^M=d_v / h$, $d=d_q$.
For brevity, we set $d_q=d_v=d$, $d_q^M=d_v^M=d / h$ throughout the rest of the paper.
Unless otherwise specified, we use a scaled softmax $\omega_j(\cdot) = \mathrm{softmax}(\cdot/\sqrt{d})$, 
which our experiments were worked robustly in most settings.

% main method section
\section{Set Transformer}
\label{sec:set_transformer}

In this section, we motivate and describe the \textit{Set Transformer}: an attention-based neural network 
that is designed to process sets of data.
Similar to other architectures, a Set Transformer consists of an encoder followed by a decoder (\textit{cf.} Section~\ref{subsec:pooling}), 
but a distinguishing feature is that each layer in the encoder and decoder attends to their inputs to produce activations.
Additionally, instead of a fixed pooling operation such as $\mathrm{mean}$, our aggregating function $\mathrm{pool}(\cdot)$ is parameterized and can thus adapt to the problem at hand.

\subsection{Permutation Equivariant (Induced) Set Attention Blocks}
We begin by defining our attention-based set operations, which we call SAB and ISAB.
While existing pooling methods for sets obtain instance features independently of other instances, we use self-attention to concurrently encode the whole set.
This gives the Set Transformer the ability to compute pairwise as well as higher-order interactions among instances during the encoding process.
For this purpose, we adapt the multihead attention mechanism used in Transformer.
We emphasize that all blocks introduced here are neural network blocks with their own parameters, and not fixed functions.

Given matrices $X, Y \in \R^{n\times d}$ which represent two sets of $d$-dimensional vectors, we define the Multihead Attention Block (MAB) with parameters $\omega$ as follows:
\[
&\mathrm{MAB}(X, Y) = \mathrm{LayerNorm}(H + \mathrm{rFF}(H)), \label{eq:mab1}\\
&\textrm{where} \ \ H = \mathrm{LayerNorm}(X + \mathrm{Multihead}(X, Y, Y; \omega)), \label{eq:mab2}
\]
$\mathrm{rFF}$ is any row-wise feedforward layer (i.e., it processes each instance independently and identically), and $\mathrm{LayerNorm}$ is layer normalization~\citep{Ba2016}.
The MAB is an adaptation of the encoder block of the Transformer~\citep{Vaswani2017} without positional encoding and dropout.
Using the MAB, we define the Set Attention Block (SAB) as
\[
\mathrm{SAB}(X) \coloneqq \mathrm{MAB}(X, X).
\label{eq:sab}
\]
In other words, an SAB takes a set and performs self-attention between the elements in the set, resulting in a set of equal size.
Since the output of SAB contains information about pairwise interactions among the elements in the input set $X$, we can stack multiple SABs to encode higher order interactions.
Note that while the SAB \eqref{eq:sab} involves a multihead attention operation \eqref{eq:mab2}, where $Q=K=V=X$, it could reduce to applying a residual block on $X$.
In practice, it learns more complicated functions due to linear projections of $X$ inside attention heads, \eqref{eq:att} and \eqref{eq:multihead2}.

 A potential problem with using SABs for set-structured data is the quadratic time complexity $\calO(n^2)$, which may be too expensive for large sets ($n \gg 1$).
We thus introduce the \textit{Induced Set Attention Block} (ISAB), which bypasses this problem. Along with the set $X\in\R^{n\times d}$, additionally define $m$ $d$-dimensional vectors $I \in \R^{m\times d}$, which we call \textit{inducing points}.
Inducing points $I$ are part of the ISAB itself, and they are \emph{trainable parameters} which we train along with other parameters of the network.
An ISAB with $m$ inducing points $I$ is defined as:
\[
\mathrm{ISAB}_m(X) &= \mathrm{MAB}(X, H) \in \R^{n \times d}, \label{eq:isab1}\\
\textrm{where} \ \ H &= \mathrm{MAB}(I, X) \in \R^{m \times d}.\label{eq:isab2}
\]
The ISAB first transforms $I$ into $H$ by attending to the input set.
The set of transformed inducing points $H$, which contains information about the input set $X$, is again attended to by the input set $X$ to finally produce a set of $n$ elements.
This is analogous to low-rank projection or autoencoder models, where inputs ($X$) are first projected onto a low-dimensional object ($H$) and then reconstructed to produce outputs.
The difference is that the goal of these methods is reconstruction whereas ISAB aims to obtain good features for the final task.
We expect the learned inducing points to encode some global structure which helps explain the inputs $X$. For example, in the amortized clustering problem on a 2D plane, the inducing points could be appropriately distributed points on the 2D plane so that the encoder can compare elements in the query dataset indirectly through their proximity to these grid points.

Note that in \eqref{eq:isab1} and \eqref{eq:isab2}, attention was computed between a set of size $m$ and a set of size $n$.
Therefore, the time complexity of $\mathrm{ISAB}_m(X;\lambda)$ is $\calO(nm)$ where $m$ is a (typically small) hyperparameter --- an improvement over the quadratic complexity of the SAB.
We also emphasize that both of our set operations (SAB and ISAB) are \emph{permutation equivariant} (definition in Section~\ref{subsec:pooling}):
\begin{property}
Both $\mathrm{SAB}(X)$ and $\mathrm{ISAB}_m(X)$ are permutation equivariant.
\end{property}

\subsection{Pooling by Multihead Attention}
A common aggregation scheme in permutation invariant networks is a dimension-wise average or maximum of the feature vectors (\textit{cf.} Section~\ref{sec:introduction}).
We instead propose to aggregate features by applying multihead attention on a learnable set of $k$ seed vectors $S \in \R^{k \times d}$.
Let $Z \in \R^{n\times d}$ be the set of features constructed from an encoder.
\textit{Pooling by Multihead Attention} (PMA) with $k$ seed vectors is defined as
\[
\mathrm{PMA}_k(Z) = \mathrm{MAB}(S, \mathrm{rFF}(Z)).
\]
Note that the output of $\mathrm{PMA}_k$ is a set of $k$ items.
We use one seed vector ($k=1$) in most cases, but for problems such as amortized clustering which requires $k$ correlated outputs, the natural thing to do is to use $k$ seed vectors.
To further model the interactions among the $k$ outputs, we apply an SAB afterwards:
\[
H = \mathrm{SAB}(\mathrm{PMA}_k(Z)).
\]
We later empirically show that such self-attention after pooling helps in modeling explaining-away (e.g., among clusters in an amortized clustering problem).

Intuitively, feature aggregation using attention should be beneficial because the influence of each instance on the target is not necessarily equal.
For example, consider a problem where the target value is the maximum value of a set of real numbers.
Since the target can be recovered using only a single instance (the largest), finding and attending to that instance during aggregation will be advantageous.

\subsection{Overall Architecture}
Using the ingredients explained above, we describe how we would construct a set transformer consists of an encoder and a decoder.
The encoder $\mathrm{Encoder}: X \mapsto Z \in \R^{n\times d}$ is a stack of SABs or ISABs, for example:
\[
\mathrm{Encoder}(X) &= \mathrm{SAB}(\mathrm{SAB}(X))\label{eq:enc_sab} \\
\mathrm{Encoder}(X) &= \mathrm{ISAB}_{m}(\mathrm{ISAB}_{m}(X)). \label{eq:enc_isab}
\]
We point out again that the time complexity for $\ell$ stacks of SABs and ISABs are $\calO(\ell n^2)$ and $\calO(\ell nm)$, respectively.
This can result in much lower processing times when using ISAB (as compared to SAB), while still maintaining high representational power.
After the encoder transforms data $X \in \R^{n \times d_x}$ into features $Z \in \R^{n\times d}$,
the decoder aggregates them into a single or a set of vectors which is fed into a feed-forward network to get final outputs.
Note that PMA with $k > 1$ seed vectors should be followed by SABs to model the correlation between $k$ outputs.
\[
&\mathrm{Decoder}(Z;\lambda) = \mathrm{rFF}(\mathrm{SAB}(\mathrm{PMA}_k(Z))) \in \R^{k \times d} \\
&\textrm{where} \ \ \mathrm{PMA}_k(Z) = \mathrm{MAB}(S, \mathrm{rFF}(Z)) \in \R^{k \times d},
\]

\subsection{Analysis}
\label{subsec:analysis}
Since the blocks used to construct the encoder (i.e., SAB, ISAB) are permutation equivariant, the mapping of the encoder $X \rightarrow Z$ is permutation equivariant as well.
Combined with the fact that the PMA in the decoder is a permutation invariant transformation, we have the following:
\begin{proposition}
The Set Transformer is permutation invariant.
\end{proposition}

Being able to approximate any function is a desirable property, especially for black-box models such as deep neural networks.
Building on previous results about the universal approximation of permutation invariant functions, we prove the universality of Set Transformers:
\begin{proposition}
The Set Transformer is a universal approximator of permutation invariant functions.
\end{proposition}
\begin{proof}
See supplementary material.
\end{proof}

\section{Related Works}
\label{sec:related_works}
\textbf{Pooling architectures for permutation invariant mappings} \hspace{1pt}
Pooling architectures for sets have been used in various problems such as 3D shape recognition~\citep{BaoguangShi2015, Su2015},
discovering causality~\citep{Lopez-Paz2016},
learning the statistics of a set~\citep{Edwards2017},
few-shot image classification~\citep{Snell2017},
and conditional regression and classification~\citep{Garnelo2018}.
\citet{Zaheer2017} discuss the structure in general and provides a partial proof of the universality of the pooling architecture,
and \citet{Wagstaff2019} further discuss the limitation of pooling architectures. \citet{Bloem-Reddy2019} provides a link between
probabilistic exchangeability and pooling architectures.

\textbf{Attention-based approaches for sets} \hspace{1pt}
Several recent works have highlighted the competency of attention mechanisms in modeling sets.
\citet{Vinyals2016a} pool elements in a set by a weighted average with weights computed using an attention mechanism.
\citet{Yang2018} propose AttSets for multi-view 3D reconstruction, where dot-product attention is applied to compute the weights used to pool the encoded features via weighted sums.
Similarly, \citet{Ilse2018} use attention-based weighted sum-pooling for multiple instance learning.
Compared to these approaches, ours use multihead attention in aggregation, and more importantly, we propose to apply self-attention after pooling to model correlation among multiple outputs.
PMA with $k=1$ seed vector and single-head attention roughly corresponds to these previous approaches.
Although not permutation invariant, \citet{Mishra2018} has attention as one of its core components to meta-learn to solve various tasks using sequences of inputs.
\citet{Kim2019} proposed attention-based conditional regression, where self-attention is applied to the query sets.

\textbf{Modeling interactions between elements in sets} \hspace{1pt}
An important reason to use the Transformer is to explicitly model higher-order interactions among the elements in a set.
\citet{Santoro2017} propose the relational network, a simple architecture that sum-pools all pairwise interactions of elements in a given set, but not higher-order interactions.
Similarly to our work, \citet{Ma2018} use the Transformer to model interactions between the objects in a video.
They use mean-pooling to obtain aggregated features which they fed into an LSTM.

\textbf{Inducing point methods} \hspace{1pt}
The idea of letting trainable vectors $I$ directly interact with data points is loosely based on the inducing point methods used in sparse Gaussian processes~\citep{Snelson2005} and the Nystr\"{o}m method for matrix decomposition~\citep{Fowlkes2004}.
$m$ trainable inducing points can also be seen as $m$ independent memory cells accessed with an attention mechanism.
The differential neural dictionary~\citep{Pritzel2017} stores previous experience as key-value pairs and uses this to process queries.
One can view the ISAB is the inversion of this idea, where queries $I$ are stored and the input features are used as key-value pairs.

%\ak{we might want to cite https://arxiv.org/abs/1805.00613 as a different approach to the problem: instead of considering permutation invariance, we could just treat permutation as a latent variable}

\section{Experiments}
\label{sec:experiments}

To evaluate the Set Transformer, we apply it to a suite of tasks involving sets of data points.
%We applied the Set Transformer to various problems that use sets of data.
We repeat all experiments five times and report performance metrics evaluated on corresponding test datasets.
Along with baselines, we compared various architectures arising from the combination of the choices of having attention in encoders and decoders. %, each of which roughly represents existing works as its special cases.
Unless specified otherwise, ``simple pooling" means average pooling.
\begin{itemize}
\setlength\itemsep{-0.4em}
\item rFF + Pooling \citep{Zaheer2017}: rFF layers in encoder and simple pooling + rFF layers in decoder.
\item rFFp-mean/rFFp-max + Pooling \citep{Zaheer2017}: rFF layers with permutation equivariant variants in encoder \citep[\eqref{eq:multihead1}]{Zaheer2017} and simple pooling + rFF layers in decoder.
\item rFF + Dotprod \citep{Yang2018,Ilse2018}: rFF layers in encoder and dot product attention based weighted sum pooling + rFF layers in decoder.
\item SAB (ISAB) + Pooling (ours):  Stack of SABs (ISABs) in encoder and simple pooling + rFF layers in decoder.
\item rFF + PMA (ours): rFF layers in encoder and PMA (followed by stack of SABs) in decoder.
\item SAB (ISAB) + PMA (ours): Stack of SABs (ISABs) in encoder and PMA (followed by stack of SABs) in decoder.
\end{itemize}

\subsection{Toy Problem: Maximum Value Regression}
\begin{table}[t]
\centering
\small
\caption{Mean absolute errors on the max regression task.}
\vspace{5pt}
\begin{tabular}{@{}ccccc@{}}
\toprule
Architecture & MAE \\\midrule
rFF + Pooling ($\mathrm{mean}$) & 2.133 $\pm$ 0.190 \\
rFF + Pooling ($\mathrm{sum}$) & 1.902 $\pm$ 0.137 \\
rFF + Pooling ($\mathrm{max}$) & \bf 0.1355 $\pm$ 0.0074\\
\midrule
SAB + PMA (ours) & 0.2085 $\pm$ 0.0127\\
\bottomrule
\label{table:max}
\end{tabular}
\end{table}
\label{subsec:max_regression}
To demonstrate the advantage of attention-based set aggregation over simple pooling operations, we consider a toy problem: regression to the maximum value of a given set.
Given a set of real numbers $\{ x_1, \ldots, x_n\}$, the goal is to return $\mathrm{max}(x_1, \cdots, x_n)$.
Given prediction $p$, we use the mean absolute error $|p - \mathrm{max}(x_1, \cdots, x_n)|$ as the loss function. We constructed simple pooling architectures with three different pooling operations: $\operatorname{max}$, $\operatorname{mea}$n, and $\operatorname{sum}$. We report loss values after training in Table~\ref{table:max}.
%Note that even though the max pooling architecture only has to learn the identity function to predict the output perfectly, Set Transformers achieve comparable performance.
%This is likely because Set Transformers can learn to find and attend to the maximum element.
%Note that the mean- and sum-pooling architectures have far higher mean absolute errors (MAE), which indicates
Mean- and sum-pooling architectures result in a high mean absolute error (MAE).
%The model with max-pooling only has to learn the identity function as its encoder to predict the output perfectly, and thus achieves the highest performance.
The model with max-pooling can predict the output perfectly by learning its encoder to be an identity function, and thus achieves the highest performance.
Notably, the Set Transformer achieves performance comparable to the max-pooling model, which underlines the importance of additional flexibility granted by attention mechanisms --- it can learn to find and attend to the maximum element.

%Our analysis of this is that learning interactions among items is difficult for aggregation functions such as sum and mean.
%The max regression problem requires knowledge of such interactions because the maximum value of a set $\{ x_1, \cdots, x_n\}$ only depends on one maximum element and thus the network must at least learn how to compare items.
%We would like to point out that most pooling architectures for sets use mean-pooling.

\begin{figure}[t]
\centering\includegraphics[width=\columnwidth]{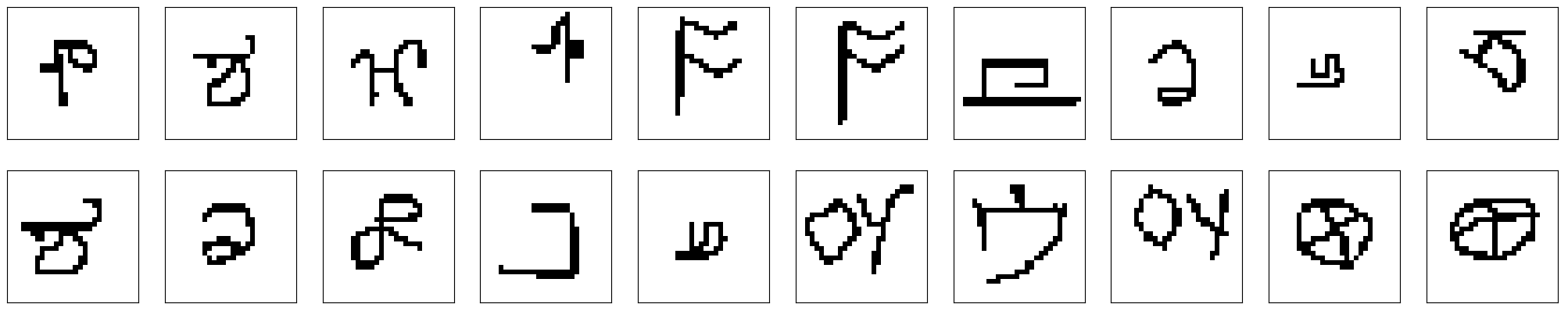}
\caption{Counting unique characters: this is a randomly sampled set of 20 images from the Omniglot dataset. There are 14 different characters inside this set.}
\label{fig:omni_try}
\end{figure}

\begin{table}[t]
\centering
\small
\caption{Accuracy on the unique character counting task.}
\vspace{5pt}
\begin{tabular}{@{}cc@{}} \toprule
Architecture & Accuracy \\ \midrule
rFF + Pooling & 0.4382 $\pm$ 0.0072 \\
rFFp-mean + Pooling & 0.4617 $\pm$ 0.0076 \\
rFFp-max + Pooling & 0.4359 $\pm$ 0.0077 \\
rFF + Dotprod & 0.4471 $\pm$ 0.0076 \\
\midrule
rFF + PMA (ours) & 0.4572 $\pm$ 0.0076\\
SAB + Pooling (ours) & 0.5659 $\pm$ 0.0077 \\
SAB + PMA (ours) & \bf 0.6037 $\pm$ 0.0075 \\
\bottomrule
\label{table:unique}
\end{tabular}
\end{table}
\begin{figure}[t]
\centering
\includegraphics[width=0.90\linewidth]{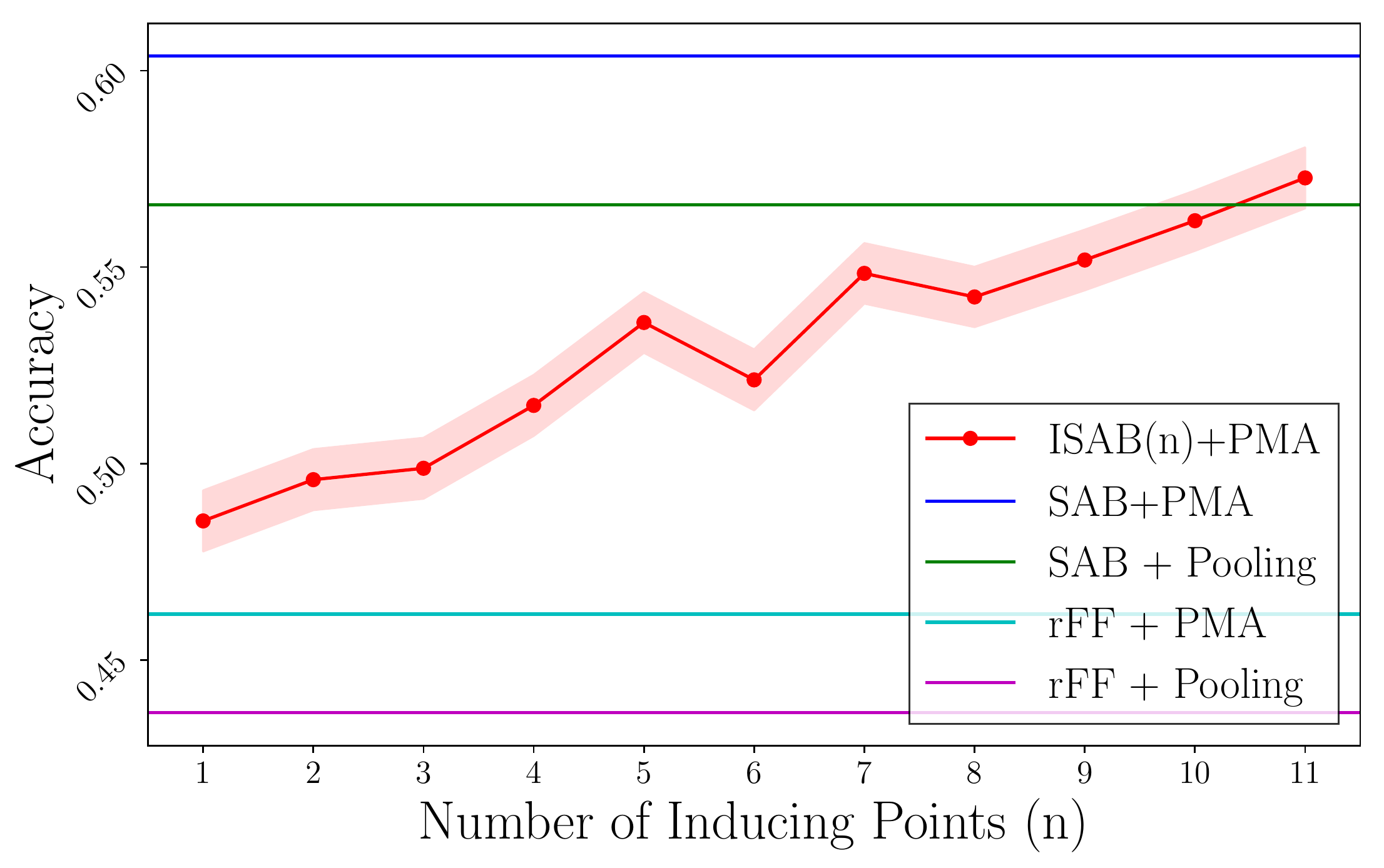}
\caption{
Accuracy of $\mathrm{ISAB}_n + \mathrm{PMA}$ on the unique character counting task.
x-axis is $n$ and y-axis is accuracy.
}
\label{fig:isab_n}
\end{figure}

\begin{table*}[h]
	\centering
	\small
	\caption{Meta clustering results.
	The number inside parenthesis indicates the number of inducing points used in ISABs of encoders.
	We show average likelihood per data for the synthetic dataset and the adjusted rand index (ARI) for the CIFAR-100 experiment.
	LL1/data, ARI1 are the evaluation metrics after a single EM update step.
	The oracle for the synthetic dataset is the log likelihood of the actual parameters used to generate the set, and the CIFAR oracle was computed by running EM until convergence.
	%The results measured from function outputs (LL0/data, ARI0) and the results after single EM update (LL1/data, ARI1) are presented.
	}
	\vspace{5pt}
	\begin{tabular}{@{}ccccc@{}}\toprule
		& \multicolumn{2}{c}{Synthetic} & \multicolumn{2}{c}{CIFAR-100} \\
		\cmidrule{2-3}
		\cmidrule{4-5}
	Architecture & LL0/data & LL1/data & ARI0 & ARI1  \\
	\midrule
	Oracle & -1.4726 & & 0.9150 &\\
	rFF + Pooling & -2.0006 $\pm$ 0.0123 &-1.6186 $\pm$ 0.0042 & 0.5593 $\pm$ 0.0149 & 0.5693 $\pm$ 0.0171 \\
	rFFp-mean + Pooling & -1.7606 $\pm$ 0.0213 &-1.5191 $\pm$ 0.0026 & 0.5673 $\pm$ 0.0053 &0.5798 $\pm$ 0.0058 \\
	rFFp-max + Pooling & -1.7692 $\pm$ 0.0130 &-1.5103 $\pm$ 0.0035 & 0.5369 $\pm$ 0.0154 &0.5536 $\pm$ 0.0186 \\
	rFF + Dotprod & -1.8549 $\pm$ 0.0128 &-1.5621 $\pm$ 0.0046 & 0.5666 $\pm$ 0.0221 &0.5763 $\pm$ 0.0212 \\
    \midrule
	SAB + Pooling (ours) & -1.6772 $\pm$ 0.0066 &-1.5070 $\pm$ 0.0115 & 0.5831 $\pm$ 0.0341 &0.5943 $\pm$ 0.0337\\
	ISAB (16) + Pooling (ours) & -1.6955 $\pm$ 0.0730 &-1.4742 $\pm$ 0.0158 & 0.5672 $\pm$ 0.0124 &0.5805 $\pm$ 0.0122\\
	rFF + PMA (ours) & -1.6680 $\pm$ 0.0040 &-1.5409 $\pm$ 0.0037 &  0.7612 $\pm$ 0.0237 &0.7670 $\pm$ 0.0231\\
	SAB + PMA (ours) & -1.5145 $\pm$ 0.0046 &-1.4619 $\pm$ 0.0048 &  0.9015 $\pm$ 0.0097 &0.9024 $\pm$ 0.0097\\
	ISAB (16) + PMA (ours) & \bf -1.5009 $\pm$ 0.0068 & \bf -1.4530 $\pm$ 0.0037 & \bf 0.9210 $\pm$ 0.0055 &\bf 0.9223 $\pm$ 0.0056 \\
	\bottomrule
	\end{tabular}
	\label{tab:meta_clustering}
\end{table*}

\begin{figure*}[t]
	\centering
	\includegraphics[width=0.4\linewidth]{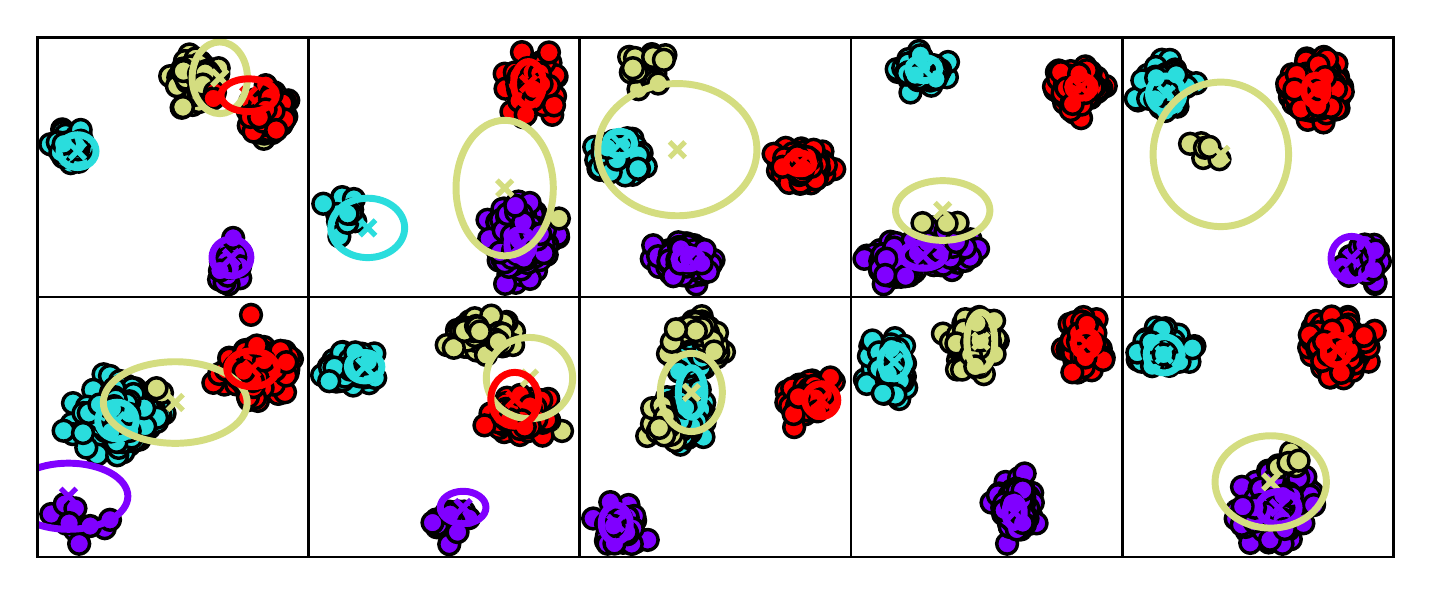}
	\includegraphics[width=0.4\linewidth]{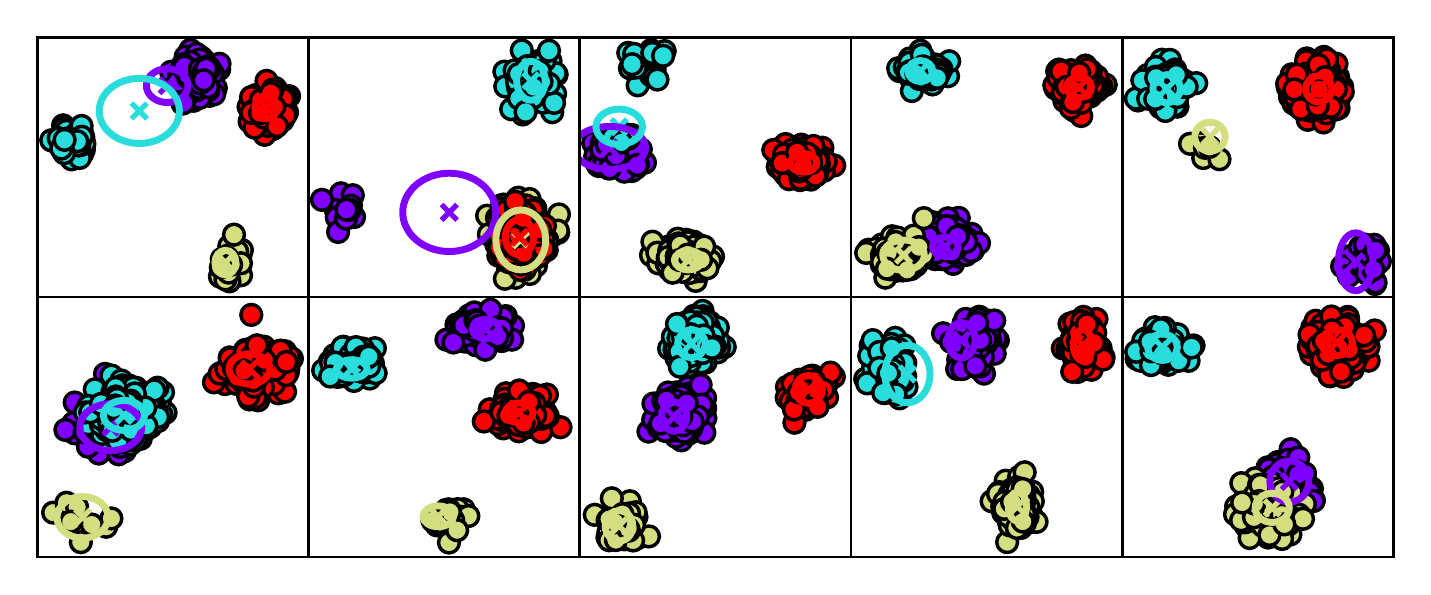}
	\includegraphics[width=0.4\linewidth]{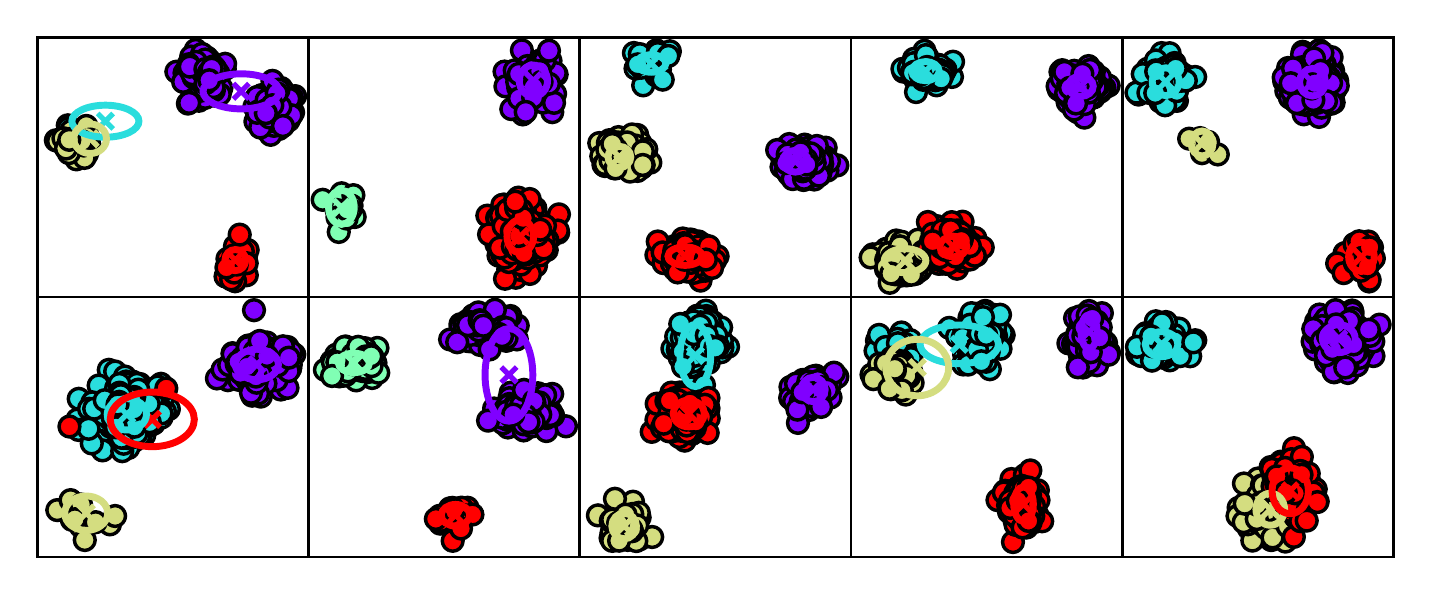}
	\includegraphics[width=0.4\linewidth]{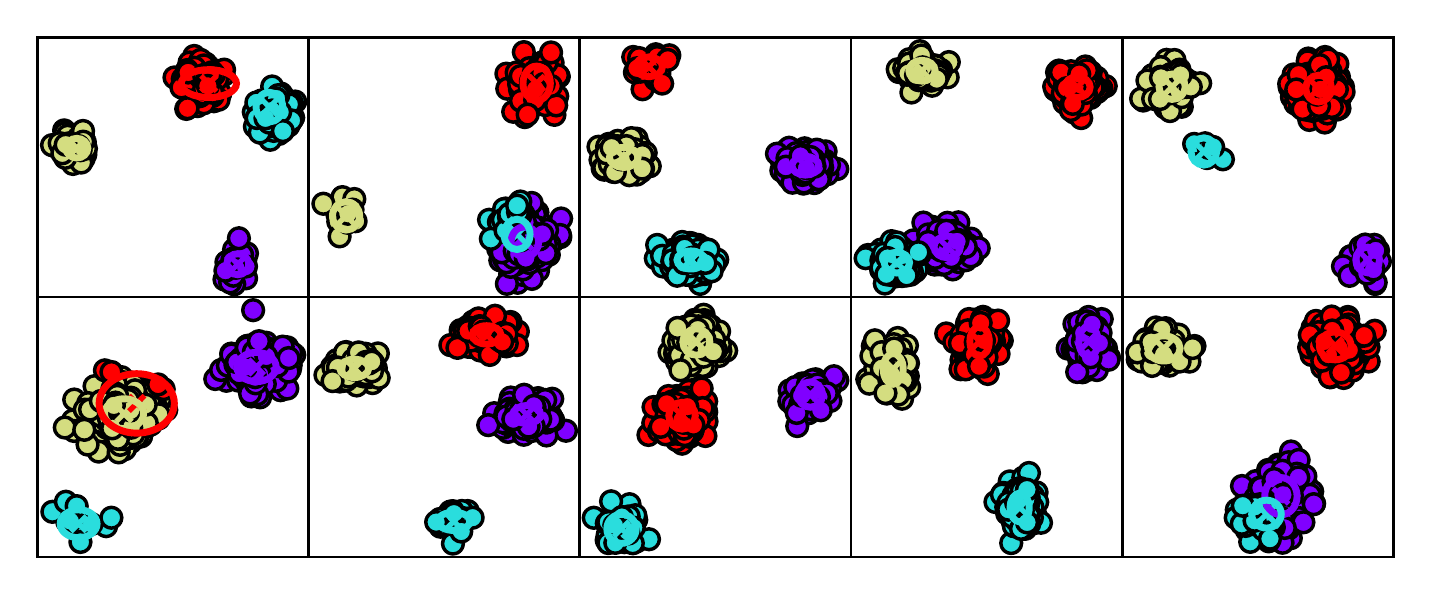}
	\caption{
    Clustering results for 10 test datasets, along with centers and covariance matrices.
    rFF+Pooling (top-left), SAB+Pooling (top-right), rFF+PMA (bottom-left), Set Transformer (bottom-right).
    Best viewed magnified in color.}
	\label{fig:synthetic_clustering}
\end{figure*}

\subsection{Counting Unique Characters}
\label{subsec:unique}
In order to test the ability of modelling interactions between objects in a set, we introduce a new task of counting unique elements in an input set.
We use the Omniglot~\citep{Lake2015} dataset, which consists of 1,623 different handwritten characters from various alphabets, where each character is represented by 20 different images.

We split all characters (and corresponding images) into train, validation, and test sets and only train using images from the train character classes.
We generate input sets by sampling between 6 and 10 images and we train the model  to predict the number of different characters inside the set.
We used a Poisson regression model to predict this number, with the rate $\lambda$ given as the output of a neural network.
We maximized the log likelihood of this model using stochastic gradient ascent.

We evaluated model performance using sets of images sampled from the test set of characters.
%Therefore, not only has the network not seen any of the images inside a given test input, but it has not even seen any image of the same character class during training.
Table~\ref{table:unique} reports accuracy, measured as the frequency at which the mode of the Poisson distribution chosen by the network is equal to the number of characters inside the input set.

We additionally performed experiments to see how the number of incuding points affects performance.
We trained $\mathrm{ISAB}_n + \mathrm{PMA}$ on this task while varying the number of inducing points ($n$).
Accuracies are shown in Figure \ref{fig:isab_n}, where other architectures are shown as horizontal lines for comparison.
Note first that even the accuracy of $\mathrm{ISAB}_1 + \mathrm{PMA}$ surpasses that of both $\mathrm{rFF} + \mathrm{Pooling}$ and $\mathrm{rFF} + \mathrm{PMA}$, and that performance tends to increase as we increase $n$.

\subsection{Amortized Clustering with Mixture of Gaussians}
%We ran all the algorithms five times and measured the averages.
We applied the set-input networks to the task of maximum likelihood of mixture of Gaussians (MoGs).
The log-likelihood of a dataset $X = \{x_1, \dots, x_n\}$ generated from an MoG with $k$ components is
{
\[
\log p(X; \theta)
= \sum_{i=1}^n  \log \sum_{j=1}^k \pi_j \gN ( x_i ; \mu_j, \mathrm{diag}(\sigma^2_j)).
\]}
The goal is to learn the optimal parameters $\theta^*(X) = \argmax_\theta \log p(X;\theta)$.
The typical approach to this problem is to run an iterative algorithm such as Expectation-Maximisation (EM) until convergence.
Instead, we aim to learn a generic meta-algorithm that directly maps the input set $X$ to $\theta^*(X)$.
One can also view this as amortized maximum likelihood learning.
Specifically, given a dataset $X$, we train a neural network to output parameters $f(X;\lambda) = \{\pi(X), \{\mu_j(X), \sigma_j(X)\}_{j=1}^k\}$ which maximize
{\small
\[
\E_X \left[ \sum_{i=1}^{|X|} \log \sum_{j=1}^k \pi_j(X) \gN(x_i ; \mu_j(X), \mathrm{diag}(\sigma_j^2(X)))\right].
\]}
We structured $f(\cdot;\lambda)$ as a set-input neural network and learned its parameters $\lambda$ using stochastic gradient ascent, where we approximate gradients using minibatches of \emph{datasets}.

We tested Set Transformers along with other set-input networks on two datasets.
We used four seed vectors for the PMA ($S\in\R^{4\times d}$) so that each seed vector generates the parameters of a cluster.

\textbf{Synthetic 2D mixtures of Gaussians}:  Each dataset contains $n \in [100, 500]$ points on a 2D plane, each sampled from one of four Gaussians.
%Average log-likelihhod/data over 1,000 test sets were compared.
%The oracle was computed by measuring log-likelihood w.r.t. the true parameters used to generate test sets.

\textbf{CIFAR-100}: Each dataset contains $n \in [100, 500]$ images sampled from four random classes in the CIFAR-100 dataset.
Each image is represented by a 512-dim vector obtained from a pretrained VGG network~\citep{Simonyan2014}.
%Adjusted Rand Index (ARI) over 1,000 test sets were compared, and the oracle was computed by running EM for each test set independently.

\begin{table*}[t]
\centering
\small
\caption{
Test accuracy for the point cloud classification task using $100, 1000, 5000$ points.
}
\vspace{5pt}
\begin{tabular}{@{}ccccc@{}}
\toprule
Architecture & 100 pts & 1000 pts & 5000 pts \\
\midrule
rFF + Pooling \citep{Zaheer2017} & - & 0.83 $\pm$ 0.01 & - \\
rFFp-max + Pooling \citep{Zaheer2017} & 0.82 $\pm$ 0.02  & 0.87 $\pm$ 0.01& \bf 0.90 $\pm$ 0.003 \\
\midrule
rFF + Pooling& 0.7951 $\pm$ 0.0166  & 0.8551 $\pm$ 0.0142 & 0.8933 $\pm$ 0.0156  \\
\midrule
rFF + PMA (ours) & 0.8076 $\pm$ 0.0160 & 0.8534 $\pm$ 0.0152  & 0.8628 $\pm$ 0.0136\\
ISAB (16) + Pooling (ours) & 0.8273 $\pm$ 0.0159 & \bf 0.8915 $\pm$ 0.0144 & \bf 0.9040 $\pm$ 0.0173  \\
ISAB (16) + PMA (ours) & \bf 0.8454 $\pm$ 0.0144 & 0.8662 $\pm$ 0.0149 & 0.8779 $\pm$ 0.0122  \\
\bottomrule
\end{tabular}
\label{table:pointcloud}
\end{table*}

We report the performance of the oracle along with the set-input neural networks in Table~\ref{tab:meta_clustering}.
We additionally report scores of all models after a single EM update.
Overall, the Set Transformer found accurate parameters and even outperformed the oracles after a single EM update.
This may be due to the relatively small size of the input sets; some clusters have fewer than 10 points.
In this regime, sample statistics can differ substantially from population statistics,
which limits the performance of the oracle while the Set Transformer can adapt accordingly.
Notably, the Set Transformer with only 16 inducing points showed the best performance, even outperforming the full Set Transformer.
We believe this is due to the knowledge transfer and regularization via inducing points, helping the network to learn global structures.
Our results also imply that the improvement from using the PMA is more significant than that of the SAB, supporting our claim of the importance of attention-based decoders.
We provide detailed generative processes, network architectures, and training schemes along with additional experiments with various numbers of inducing points in the supplementary material.

\subsection{Set Anomaly Detection}
\label{subsec:meta_anomaly}

\begin{figure}[t]
\centering
\includegraphics[width=0.90\linewidth]{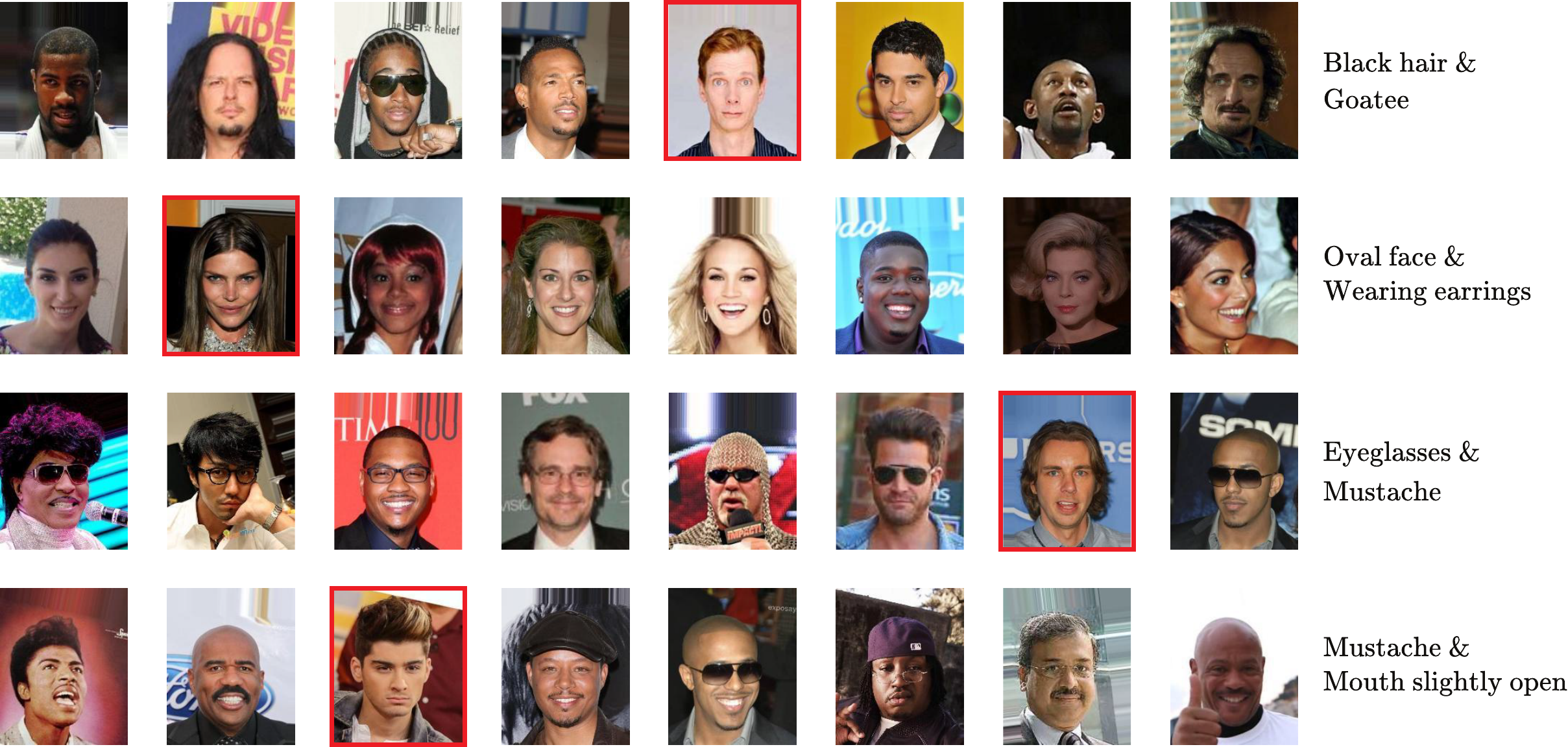}
\caption{
Sampled datasets.
Each row is a dataset, consisting of 7 normal images and 1 anomaly (red box).
In each subsampled dataset, a normal image has two attributes (rightmost column) which anomalies do not.
}
\label{fig:celeba_figures}
\end{figure}

\begin{table}[t]
\centering
\small
\caption{Meta set anomaly results.
Each architecture is evaluated using average of test AUROC and test AUPR.
%Random guess shows AUROC as 0.5 and AUPR as 0.125.
%Other describing methods are same with the above description.
%200 test datasets are used to measure the performance of each model and all experiments are repeated 5 times.
}
\vspace{5pt}
\begin{tabular}{@{}ccccc@{}}\toprule
Architecture & Test AUROC & Test AUPR  \\
\midrule
Random guess & 0.5 & 0.125 \\
rFF + Pooling & 0.5643 $\pm$ 0.0139 & 0.4126 $\pm$ 0.0108 \\
rFFp-mean + Pooling & 0.5687 $\pm$ 0.0061 & 0.4125 $\pm$ 0.0127 \\
rFFp-max + Pooling & 0.5717 $\pm$ 0.0117 & 0.4135 $\pm$ 0.0162 \\
rFF + Dotprod & 0.5671 $\pm$ 0.0139 & 0.4155 $\pm$ 0.0115 \\
\midrule
SAB + Pooling (ours) & 0.5757 $\pm$ 0.0143 & 0.4189 $\pm$ 0.0167 \\
%ISAB (4) + Pooling & 0.5679 $\pm$ 0.0155 & 0.4202 $\pm$ 0.0166 \\
rFF + PMA (ours) & 0.5756 $\pm$ 0.0130 & 0.4227 $\pm$ 0.0127 \\
SAB + PMA (ours) & \bf 0.5941 $\pm$ 0.0170 & \bf 0.4386 $\pm$ 0.0089 \\
%Set Transformer (4) & 0.5629 $\pm$ 0.0167 & 0.4122 $\pm$ 0.0151 \\
\bottomrule
\end{tabular}
\label{tab:meta_anomaly}
\end{table}
We evaluate our methods on the task of meta-anomaly detection within a set using the CelebA dataset.
The dataset consists of 202,599 images with the total of 40 attributes.
We randomly sample 1,000 sets of images.
For every set, we select two attributes at random and construct the set by selecting seven images containing both attributes and one image with neither.
The goal of this task is to find the image that does not belong to the set.
We give a detailed description of the experimental setup in the supplementary material.
We report the area under receiver operating characteristic curve (AUROC) and area under precision-recall curve (AUPR) in Table~\ref{tab:meta_anomaly}.
Set Transformers outperformed all other methods by a significant margin.

\subsection{Point Cloud Classification}

We evaluated Set Transformers on a classification task using the ModelNet40~\citep{Chang2015} dataset\footnote{The point-cloud dataset used in this experiment was obtained directly from the authors of \citet{Zaheer2017}.}, which contains three-dimensional objects in $40$ different categories.
Each object is represented as a point cloud, which we treat as a set of $n$ vectors in $\R^3$.
We performed experiments with input sets of size $n \in \{100, 1000, 5000\}$.
Because of the large set sizes, MABs are prohibitively time-consuming due to their $\calO(n^2)$ time complexity.

Table~\ref{table:pointcloud} shows classification accuracies.
We point out that \citet{Zaheer2017} used significantly more engineering for the $5000$ point experiment.
For this experiment only, they augmented data (scaling, rotation) and used a different optimizer (Adamax) and learning rate schedule.
Set Transformers were superior when given small sets, but were outperformed by ISAB (16) + Pooling on larger sets.
First note that classification is harder when given fewer points.
We think Set Transformers were outperformed in the problems with large sets because such sets already had sufficient information for classification, diminishing the need to model complex interactions among points.
We point out that PMA outperformed simple pooling in all other experiments.

\section{Conclusion}
\label{sec:conclusion}

In this paper, we introduced the Set Transformer, an attention-based set-input neural network architecture.
Our proposed method uses attention mechanisms for both encoding and aggregating features, and we have empirically validated that both of them are necessary for modelling complicated interactions among elements of a set.
We also proposed an inducing point method for self-attention, which makes our approach scalable to large sets.
We also showed useful theoretical properties of our model, including the fact that it is a universal approximator for permutation invariant functions.
An interesting future work would be to apply Set Transformers to meta-learning problems. In particular, using Set Transformers to meta-learn posterior inference in Bayesian models seems like a promising line of research. Another exciting extension of our work would be to model the uncertainty in set functions by injecting noise variables into Set Transformers in a principled way.

\paragraph{Acknowledgments}
JL and YWT's research leading to these results has received funding from the
European Research Council under the European Union's Seventh Framework
Programme (FP7/2007-2013) ERC grant agreement no. 617071.
JL has also received funding from EPSRC under grant EP/P026753/1. JL acknowledges support from IITP grant funded by the Korea government(MSIT) (No.2017-0-01779, XAI) and Samsung Research Funding \& Incubation Center of Samsung Electronics under Project Number SRFC-IT1702-15.

\bibliography{set_transformer}
\bibliographystyle{icml2019}

\end{document}

\onecolumn

\icmltitle{Supplementary Material for Set Transformer}

\icmlsetsymbol{equal}{*}

\begin{icmlauthorlist}
\icmlauthor{Juho Lee}{oxford1,aitrics}
\icmlauthor{Yoonho Lee}{kakao}
\icmlauthor{Jungtaek Kim}{postech}
\icmlauthor{Adam R. Kosiorek}{oxford1,oxford2}
\icmlauthor{Seungjin Choi}{postech}
\icmlauthor{Yee Whye Teh}{oxford1}
\end{icmlauthorlist}

\icmlaffiliation{oxford1}{Department of Statistics, University of Oxford, United Kingdom}
\icmlaffiliation{oxford2}{Oxford Robotics Institute, University of Oxford, United Kingdom}
%\icmlaffiliation{oxford2}{Applied Artificial Intelligence Lab, Oxford Robotics Institute, University of Oxford}
\icmlaffiliation{postech}{Department of Computer Science and Engineering, POSTECH, Republic of Korea}
\icmlaffiliation{aitrics}{AITRICS, Republic of Korea}
\icmlaffiliation{kakao}{Kakao Corporation, Republic of Korea}

\icmlcorrespondingauthor{Juho Lee}{juho.lee@stats.ox.ac.uk}

\icmlkeywords{Set transformer, Set-taking network, Attention-based network, Permutation-invariant network}

\vskip 0.3in

\section{Proofs}
\label{app:proofs}

\begin{lemma}
The mean operator $\mathrm{mean}(\{x_1, \ldots, x_n \}) = \frac{1}{n} \sum_{i=1}^n x_i$ is a special case of dot-product attention with softmax.
\end{lemma}
\begin{proof}
Let $s = \mathbf{0} \in \R^d$ and $X \in \R^{n \times d}$.
\[
\mathrm{Att}(s, X, X; \mathrm{softmax})
= \mathrm{softmax}\left( \frac{s X^\top}{\sqrt{d}} \right) X
= \frac{1}{n} \sum_{i=1}^n x_i \nonumber
\]
\end{proof}

\begin{lemma}
The decoder of a Set Transformer, given enough nodes, can express any element-wise function of the form $\left( \frac{1}{n} \sum^n_{i=1} z_i^p \right)^\frac{1}{p}$.
\end{lemma}
\begin{proof}
We first note that we can view the decoder as the composition of functions
\[
\mathrm{Decoder}(Z) &= \mathrm{rFF}(H) \\
\textrm{where} \ \ H &=\mathrm{rFF}(\mathrm{MAB}(Z, \mathrm{rFF}(Z))) \label{eq:decoder_h}
\]
We focus on $H$ in \eqref{eq:decoder_h}.
Since feed-forward networks are universal function approximators at the limit of infinite nodes, let the feed-forward layers in front and back of the MAB encode the element-wise functions $z \rightarrow z^p$ and $z \rightarrow z^{\frac{1}{p}}$, respectively.
We let $h=d$, so the number of heads is the same as the dimensionality of the inputs, and each head is one-dimensional.
Let the projection matrices in multi-head attention ($W_j^Q, W_j^K, W_j^V$) represent projections onto the jth dimension and the output matrix ($W^O$) the identity matrix.
Since the mean operator is a special case of dot-product attention, by simple composition, we see that an MAB can express any dimension-wise function of the form
\[
M_p(z_1, \cdots, z_n) = \left( \frac{1}{n} \sum^n_{i=1} z_i^p \right)^\frac{1}{p}.
\]
\end{proof}

\begin{lemma}
\label{lem:sum}
A PMA, given enough nodes, can express sum pooling $\left( \sum^n_{i=1} z_i \right)$.
\end{lemma}
\begin{proof}
We prove this by construction.

Set the seed $s$ to a zero vector and let $\omega(\cdot)=1+f(\cdot)$, where $f$ is any activation function such that $f(0)=0$.
The identiy, sigmoid, or relu functions are suitable choices for $f$.
The output of the multihead attention is then simply a sum of the values, which is $Z$ in this case.
\end{proof}

We additionally have the following universality theorem for pooling architectures:
\begin{theorem}
\label{thm:univ}
Models of the form $\mathrm{rFF}(\mathrm{sum}(\mathrm{rFF}(\cdot)))$ are universal function approximators in the space of permutation invariant functions.
\end{theorem}
\begin{proof}
See Appendix A of \citet{Zaheer2017}.
\end{proof}

By Lemma~\ref{lem:sum}, we know that $\mathrm{decoder}(Z)$ can express any function of the form $\textrm{rFF}(\mathrm{sum}(Z))$.
Using this fact along with Theorem~\ref{thm:univ}, we can prove the universality of Set Transformers:

\begin{proposition}
\label{thm:univ_set}
The Set Transformer is a universal function approximator in the space of permutation invariant functions.
\end{proposition}
\begin{proof}
By setting the matrix $W^O$ to a zero matrix in every SAB and ISAB, we can ignore all pairwise interaction terms in the encoder.
Therefore, the $\mathrm{encoder}(X)$ can express any instance-wise feed-forward network ($Z=\mathrm{rFF}(X)$).
Directly invoking Theorem~\ref{thm:univ} concludes this proof.
\end{proof}

While this proof required us to ignore the pairwise interaction terms inside the SABs and ISABs to prove that Set Transformers are universal function approximators, our experiments indicated that self-attention in the encoder was crucial for good performance.

\section{Experiment Details}
In all implementations, we omit the feed-forward layer in the beginning of the decoder ($\mathrm{rFF}(Z)$) because the end of the previous block contains a feed-forward layer.
All MABs (inside SAB, ISAB and PMA) use fully-connected layers with ReLU activations for rFF layers.

In the architecture descriptions, $\mathrm{FC}(d,f)$ denotes the fully-connected layer with $d$ units and activation function $f$.
$\mathrm{SAB}(d, h)$ denotes the SAB with $d$ units and $h$ heads.
$\mathrm{ISAB}_m(d, h)$ denotes the ISAB with $d$ units, $h$ heads and $m$ inducing points.
$\mathrm{PMA}_k(d, h)$ denotes the PMA with $d$ units, $h$ heads and $k$ vectors. All MABs used in SAB and PMA uses FC layers with ReLU activations for FF layers.

\newcommand{\relu}{\mathrm{ReLU}}
\newcommand{\bn}{\mathrm{BN}}
\newcommand{\fc}{\mathrm{FC}}
\newcommand{\sab}{\mathrm{SAB}}
\newcommand{\isab}{\mathrm{ISAB}}
\newcommand{\pma}{\mathrm{PMA}}
\newcommand{\conv}{\mathrm{Conv}}

\subsection{Max Regression}
Given a set of real numbers $\{ x_1, \ldots, x_n\}$, the goal of this task is to return the maximum value in the set $\mathrm{max}(x_1, \cdots, x_n)$.
We construct training data as follows.
We first sample a dataset size $n$ uniformly from the set of integers $\{1, \cdots, 10 \}$.
We then sample real numbers $x_i$ independently from the interval $[0, 100]$.
Given the network's prediction $p$, we use the actual maximum value $\mathrm{max}(x_1, \cdots, x_n)$ to compute the mean absolute error $|p - \mathrm{max}(x_1, \cdots, x_n)|$.
We don't explicitly consider splits of train and test data, since we sample a new set $\{ x_1, \ldots, x_n\}$ at each time step.

\begin{table}[h]
\small
\centering
\caption{Detailed architectures used in the max regression experiments.}
\vspace{5pt}
\begin{tabular}{@{}cccc@{}}\toprule
\multicolumn{2}{c}{Encoder} & \multicolumn{2}{c}{Decoder}\\
\cmidrule(lr){1-2} \cmidrule(lr){3-4}
FF & SAB & Pooling & PMA \\
\midrule
$\fc(64,\relu)$ & $\sab(64, 4)$ & $\mathrm{mean, sum, max}$ & $\pma_1(64,4)$\\
$\fc(64,\relu)$ & $\sab(64, 4)$ & $\fc(64,\relu)$ & $\fc(1,-)$\\
$\fc(64,\relu)$ & & $\fc(1,-)$ & \\
$\fc(64,-)$ & & & \\
\bottomrule
\end{tabular}
\label{tab:max_architecture}
\end{table}

We show the detailed architectures used for the experiments in Table~\ref{tab:max_architecture}.
We trained all networks using the Adam optimizer~\citep{Kingma2015} with a constant learning rate of $10^{-3}$ and a batch size of 128 for 20,000 batches, after which loss converged for all architectures.

\subsection{Counting Unique Characters}
\iffalse
\begin{figure}[h]
\centering\includegraphics[width=\columnwidth]{figs/omni20.png}
  \caption{Try the task yourself: this is a randomly sampled set of 20 images from the Omniglot dataset. There are 14 different characters inside this set.}
\label{fig:omni_try}
\end{figure}
\fi

\begin{table}[h]
\centering
\caption{Detailed results for the unique character counting experiment.}
\vspace{5pt}
\begin{tabular}{@{}cc@{}} \toprule
Architecture & Accuracy\\ \midrule
rFF + Pooling & 0.4366 $\pm$ 0.0071 \\
rFF + PMA & 0.4617 $\pm$ 0.0073\\
rFFp-mean + Pooling & 0.4617 $\pm$ 0.0076 \\
rFFp-max + Pooling & 0.4359 $\pm$ 0.0077 \\
rFF + Dotprod & 0.4471 $\pm$ 0.0076 \\
\hline
SAB + Pooling & 0.5659 $\pm$ 0.0067 \\
SAB + Dotprod & 0.5888 $\pm$ 0.0072 \\
SAB + PMA (1) & \bf 0.6037 $\pm$ 0.0072\\
SAB + PMA (2) & 0.5806 $\pm$ 0.0075\\
SAB + PMA (4) & 0.5945 $\pm$ 0.0072\\
SAB + PMA (8)  & 0.6001 $\pm$ 0.0078\\
\bottomrule
\end{tabular}
\label{table:appunique}
\end{table}

\begin{table}[h]
\small
\centering
\caption{Detailed architectures used in the unique character counting experiments.}
\vspace{5pt}
\begin{tabular}{@{}cccc@{}}\toprule
\multicolumn{2}{c}{Encoder} & \multicolumn{2}{c}{Decoder}\\
\cmidrule(lr){1-2} \cmidrule(lr){3-4}
rFF & SAB & Pooling & PMA \\
\midrule
$\conv(64, 3, 2, \bn, \relu)$ & $\conv(64, 3, 2, \bn, \relu)$ & $\mathrm{mean}$ & $\pma_1(8, 8)$ \\
$\conv(64, 3, 2, \bn, \relu)$ & $\conv(64, 3, 2, \bn, \relu)$ & $\fc(64, \relu)$ & $\fc(1, \mathrm{softplus})$ \\
$\conv(64, 3, 2, \bn, \relu)$ & $\conv(64, 3, 2,\bn,\relu)$ & $\fc(1,\mathrm{softplus})$ & \\
$\conv(64, 3, 2, \bn, \relu)$ & $\conv(64,3,2,\bn,\relu)$ & & \\
$\fc(64,\relu)$ & $\sab(64,4)$ & & \\
$\fc(64,\relu)$ & $\sab(64,4)$ & & \\
$\fc(64,\relu)$& & & \\
$\fc(64,-)$ & & & \\
\bottomrule
\end{tabular}
\label{tab:unique_architecture}
\end{table}

The task generation procedure is as follows.
We first sample a set size $n$ uniformly from the set of integers $\{ 6, \ldots, 10\}$.
We then sample the number of characters $c$ uniformly from $ \{ 1, \ldots, n \}$.
We sample $c$ characters from the training set of characters, and randomly sample instances of each character so that the total number of instances sums to $n$ and each set of characters has at least one instance in the resulting set.

We show the detailed architectures used for the experiments in Table~\ref{tab:unique_architecture}.
For both architectures, the resulting $1$-dimensional output is passed through a $\mathrm{softplus}$ activation to produce the Poisson parameter $\gamma$.
The role of $\mathrm{softplus}$ is to ensure that $\gamma$ is always positive.

The loss function we optimize, as previously mentioned, is the log likelihood $\log p(x | \gamma) = x \log(\gamma) - \gamma - \log(x!)$.
We chose this loss function over mean squared error or mean absolute error because it seemed like the more logical choice when trying to make a real number match a target integer.
Early experiments showed that directly optimizing for mean absolute error had roughly the same result as optimizing $\gamma$ in this way and measuring $|\gamma - x|$.
We train using the Adam optimizer with a constant learning rate of $10^{-4}$ for 200,000 batches each with batch size 32.

\subsection{Solving maximum likelihood problems for mixture of Gaussians}
\label{subsec:clustering}
\subsubsection{Details for 2D synthetic mixtures of Gaussians experiment}
We generated the datasets according to the following generative process.
\begin{enumerate}
\item Generate the number of data points, $n \sim \mathrm{Unif}(100, 500)$.
\item Generate $k$ centers.
\[
\mu_{j,d} \sim \mathrm{Unif}(-4, 4), \quad j=1,\dots, 4, \,\,\, d=1,2.
 \]
\item Generate cluster labels.
\[
\pi \sim \mathrm{Dir}([1,1]^\top), \quad z_i \sim \mathrm{Categorical}(\pi), \,\,  i=1,\dots, n.
\]
\item Generate data from spherical Gaussian.
\[
x_i \sim \gN(\mu_{z_i}, (0.3)^2 I).
\]
\end{enumerate}

Table~\ref{tab:2d_synthetic_architecture} summarizes the architectures used for the experiments. For all architectures, at each training step, we generate 10 random datasets according to the above generative process, and updated the parameters via Adam optimizer  with initial learning rate $10^{-3}$. We trained all the algorithms for $50k$ steps, and decayed the learning rate to $10^{-4}$ after $35k$ steps. Table~\ref{tab:2d_synthetic_more_results} summarizes the detailed results with various number of inducing points in the ISAB.
Figure~\ref{fig:synthetic_clustering} shows the actual clustering results based on the predicted parameters.

\begin{table}[h]
\small
\centering
\caption{Detailed architectures used in 2D synthetic experiments.}
\vspace{5pt}
\begin{tabular}{@{}ccccc@{}}\toprule
\multicolumn{3}{c}{Encoder} & \multicolumn{2}{c}{Decoder}\\
\cmidrule(lr){1-3} \cmidrule(lr){4-5}
rFF & SAB & ISAB & Pooling & PMA \\
\midrule
$\fc(128,\relu)$ & $\sab(128,4)$ & $\isab_m(128,4)$ & $\mathrm{mean}$ & $\pma_4(128,4)$\\
$\fc(128,\relu)$ & $\sab(128,4)$ & $\isab_m(128,4)$ & $\fc(128,\relu)$ & $\sab(128,4)$\\
$\fc(128,\relu)$ & & & $\fc(128,\relu)$ & $\fc(4\cdot(1+2\cdot 2),-)$ \\
$\fc(128,\relu)$& & & $\fc(128,\relu)$ & \\
& & & $\fc(4\cdot(1+2\cdot 2),-)$ & \\
\bottomrule
\end{tabular}
\label{tab:2d_synthetic_architecture}
\end{table}

\begin{table}[h]
\centering
\caption{Average log-likelihood/data (LL0/data) and average log-likelihood/data after single EM iteration (LL1/data) the clustering experiment. The number inside parenthesis  indicates the number of inducing points used in the SABs of encoder. For all PMAs, four seed vectors were used.}
\vspace{5pt}
\begin{tabular}{@{}ccc@{}} \toprule
Architecture & LL0/data & LL1/data \\ \midrule
Oracle & -1.4726 & \\
\hline
rFF + Pooling& -2.0006 $\pm$ 0.0123 &-1.6186 $\pm$ 0.0042 \\
rFFp-mean + Pooling & -1.7606 $\pm$ 0.0213 &-1.5191 $\pm$ 0.0026 \\
rFFp-max + Pooling & -1.7692 $\pm$ 0.0130 &-1.5103 $\pm$ 0.0035 \\
rFF+Dotprod & -1.8549 $\pm$ 0.0128 &-1.5621 $\pm$ 0.0046 \\
\hline
SAB + Pooling & -1.6772 $\pm$ 0.0066 &-1.5070 $\pm$ 0.0115 \\
ISAB (16) + Pooling & -1.6955 $\pm$ 0.0730 &-1.4742 $\pm$ 0.0158 \\
ISAB (32) + Pooling & -1.6353 $\pm$ 0.0182 &-1.4681 $\pm$ 0.0038 \\
ISAB (64) + Pooling & -1.6349 $\pm$ 0.0429 &-1.4664 $\pm$ 0.0080 \\
rFF + PMA  & -1.6680 $\pm$ 0.0040 &-1.5409 $\pm$ 0.0037 \\
SAB + PMA  & -1.5145 $\pm$ 0.0046 &-1.4619 $\pm$ 0.0048 \\
ISAB (16) + PMA & -1.5009 $\pm$ 0.0068 &-1.4530 $\pm$ 0.0037 \\
ISAB (32) + PMA & \bf -1.4963 $\pm$ 0.0064 & \bf-1.4524 $\pm$ 0.0044 \\
ISAB (64) + PMA & -1.5042 $\pm$ 0.0158 &-1.4535 $\pm$ 0.0053 \\
\bottomrule
\end{tabular}
\label{tab:2d_synthetic_more_results}
\end{table}

\subsubsection{2D Synthetic Mixtures of Gaussians Experiment on Large-scale Data}
To show the scalability of the set transformer, we conducted additional experiments on large-scale 2D synthetic clustering dataset. We generated the synthetic data as before, except that we sample the number of data points $n~\mathrm{Unif}(1000, 5000)$ and set $k=6$. We report the clustering accuracy of a subset of comparing methods in Table~\ref{tab:2d_synthetic_large_scale}. The set transformer with only 32 inducing points works extremely well, demonstrating its scalability and efficiency.

\iffalse
\begin{figure}[h]
\centering
\includegraphics[width=0.49\linewidth]{figs/ff.pdf}
\includegraphics[width=0.49\linewidth]{figs/sab_ff.pdf}
\includegraphics[width=0.49\linewidth]{figs/ff_sab.pdf}
\includegraphics[width=0.49\linewidth]{figs/sab.pdf}
\caption{Clustering results for 10 test datasets, along with centers and covariance matrices. rFF+Pooling (top-left), SAB+Pooling (top-right), rFF+PMA (bottom-left), Set Transformer (bottom-right). Best viewed magnified in color.}
\label{fig:synthetic_clustering}
\end{figure}
\fi

\begin{table}
\centering
\caption{Average log-likelihood/data (LL0/data) and average log-likelihood/data after single EM iteration (LL1/data) the clustering experiment on large-scale data. The number inside parenthesis  indicates the number of inducing points used in the SABs of encoder. For all PMAs, six seed vectors were used.}
\vspace{5pt}
\begin{tabular}{@{}ccc@{}} \toprule
Architecture & LL0/data & LL1/data \\ \midrule
Oracle & -1.8202 & \\
\hline
rFF + Pooling & -2.5195 $\pm$ 0.0105 &-2.0709 $\pm$ 0.0062 \\
rFFp-mean + Pooling & -2.3126 $\pm$ 0.0154 &-1.9749 $\pm$ 0.0062 \\
\hline
rFF + PMA (6) & -2.0515 $\pm$ 0.0067 &-1.9424 $\pm$ 0.0047 \\
SAB (32) + PMA (6) & \bf-1.8928 $\pm$ 0.0076 & \bf-1.8549 $\pm$ 0.0024 \\
\bottomrule
\end{tabular}
\label{tab:2d_synthetic_large_scale}
\end{table}

\subsubsection{Details for CIFAR-100 amortized clutering experiment}
\label{subsec:cifar}
We pretrained VGG net~\citep{Simonyan2014} with CIFAR-100, and obtained the test accuracy 68.54\%. Then, we extracted feature vectors of 50k training images of CIFAR-100 from the 512-dimensional hidden layers of the VGG net (the layer just before the last layer). Given these feature vectors, the generative process of datasets is as follows.
\begin{enumerate}
\item Generate the number of data points, $n \sim \mathrm{Unif}(100, 500)$.
\item Uniformly sample four classes among 100 classes.
\item Uniformly sample $n$ data points among four sampled classes.
\end{enumerate}
Table~\ref{tab:cifar100_architecture} summarizes the architectures used for the experiments. For all architectures, at each training step, we generate 10 random datasets according to the above generative process, and updated the parameters via Adam optimizer with initial learning rate $10^{-4}$. We trained all the algorithms for $50k$ steps, and decayed the learning rate to $10^{-5}$ after $35k$ steps. Table~\ref{tab:cifar100_more_results} summarizes the detailed results with various number of inducing points in the ISAB.

\begin{table}
\small
\centering
\caption{Detailed architectures used in CIFAR-100 meta clustering experiments.}
\vspace{5pt}
\begin{tabular}{@{}ccccc@{}}\toprule
\multicolumn{3}{c}{Encoder} & \multicolumn{2}{c}{Decoder}\\
\cmidrule(lr){1-3} \cmidrule(lr){4-5}
rFF & SAB & ISAB & rFF & PMA \\
\midrule
$\fc(256,\relu)$ & $\sab(256, 4)$ & $\isab_m(256, 4)$ & $\mathrm{mean}$ & $\pma_4(128,4)$\\
$\fc(256,\relu)$ &  $\sab(256, 4)$ & $\isab_m(256, 4)$ & $\fc(256,\relu)$ & $\sab(256, 4)$\\
$\fc(256,\relu)$ &  $\sab(256, 4)$ & $\isab_m(256, 4)$ & $\fc(256,\relu)$ & $\sab(256, 4)$ \\
$\fc(256,\relu)$ & & & $\fc(256,\relu)$) & $\fc(4\cdot(1+2\cdot 512),-)$ \\
$\fc(256,\relu)$ & & &$\fc(256,\relu)$ & \\
$\fc(256,-)$ & & & $\fc(256,\relu)$ & \\
 & & & $\fc(4\cdot(1+2\cdot 512),-)$& \\
\bottomrule
\end{tabular}
\label{tab:cifar100_architecture}
\end{table}

\begin{table}
\centering
\caption{Average clustering accuracies measured by Adjusted Rand Index (ARI) for CIFAR100 clustering experiments. The number inside parenthesis  indicates the number of inducing points used in the SABs of encoder. For all PMAs, four seed vectors were used.
}
\vspace{5pt}
\begin{tabular}{@{}ccc@{}} \toprule
Architecture & ARI0 & ARI1 \\ \midrule
Oracle & 0.9151 & \\
rFF + Pooling & 0.5593 $\pm$ 0.0149 &0.5693 $\pm$ 0.0171 \\
rFFp-mean + Pooling & 0.5673 $\pm$ 0.0053 &0.5798 $\pm$ 0.0058 \\
rFFp-max + Pooling & 0.5369 $\pm$ 0.0154 &0.5536 $\pm$ 0.0186 \\
rFF+Dotprod & 0.5666 $\pm$ 0.0221 &0.5763 $\pm$ 0.0212 \\
\hline
SAB + Pooling & 0.5831 $\pm$ 0.0341 &0.5943 $\pm$ 0.0337 \\
ISAB (16) + Pooling  & 0.5672 $\pm$ 0.0124 &0.5805 $\pm$ 0.0122 \\
ISAB (32) + Pooling & 0.5587 $\pm$ 0.0104 &0.5700 $\pm$ 0.0134 \\
ISAB (64) +  Pooling & 0.5586 $\pm$ 0.0205 &0.5708 $\pm$ 0.0183 \\
rFF + PMA  & 0.7612 $\pm$ 0.0237 &0.7670 $\pm$ 0.0231 \\
SAB + PMA & 0.9015 $\pm$ 0.0097 &0.9024 $\pm$ 0.0097 \\
ISAB (16) + PMA & \bf 0.9210 $\pm$ 0.0055 &\bf 0.9223 $\pm$ 0.0056 \\
ISAB (32) + PMA & 0.9103 $\pm$ 0.0061 &0.9119 $\pm$ 0.0052 \\
ISAB (64) + PMA & 0.9141 $\pm$ 0.0040 &0.9153 $\pm$ 0.0041 \\
\bottomrule
\end{tabular}
\label{tab:cifar100_more_results}
\end{table}

\subsection{Set Anomaly Detection}
\label{app:anomaly}

\iffalse
\begin{figure}[h]
\centering
\includegraphics[width=0.95\linewidth]{figs/celeba.pdf}
\caption{Subsampled dataset examples. Each row is one dataset, which is composed of 7 normal images and 1 abnormal image (red box). Normal images in each subsampled dataset have both two attributes that are described in the rightmost column of figure. On the other hand, abnormal image does not contain the two attributes.}
\label{fig:celeba_figures}
\end{figure}
\fi

\begin{table}[h]
\small
\centering
\caption{Detailed architectures used in CelebA meta set anomaly experiments.
$\conv(d, k, s, r, f)$ is a convolutional layer with $d$ output channels, $k$ kernel size, $s$ stride size, $r$ regularization method, and activation function $f$. If $d$ is a list, each element in the list is distributed.
$\fc(d, f, r)$ denotes a fully-connected layer with $d$ units, activation function $f$ and $r$ regularization method.
If ${d}$ is a list, each element in the list is distributed.
$\sab(d, h)$ denotes the SAB with $d$ units and $h$ heads.
%$\mathrm{ISAB}(d, h, n_{\mathrm{ind}})$ denotes the ISAB with $d$ units, $h$ heads and $n_{\mathrm{ind}}$ inducing points.
$\pma(d, h, n_{\mathrm{seed}})$ denotes the PMA with $d$ units, $h$ heads and $n_{\mathrm{seed}}$ vectors.
All MABs used in SAB and PMA uses FC layers with ReLU activations for rFF layers.}
\vspace{5pt}
\begin{tabular}{@{}cccc@{}}\toprule
\multicolumn{2}{c}{Encoder} & \multicolumn{2}{c}{Decoder}\\
\cmidrule(lr){1-2} \cmidrule(lr){3-4}
rFF & SAB & Pooling & PMA \\
\midrule
\multicolumn{2}{c}{$\conv([32, 64, 128], 3, 2, \mathrm{Dropout}, \relu)$} & $\mathrm{mean}$ & $\pma_4(128, 4)$ \\
\multicolumn{2}{c}{$\fc([1024, 512, 256], -, \mathrm{Dropout})$} & $\fc(128, \relu, -)$ & $\sab(128, 4)$ \\
\multicolumn{2}{c}{$\fc(256, -, -)$} & $\fc(128, \relu, -)$ & $\fc(256 \cdot 8, -, -)$ \\
$\fc([128, 128, 128], \relu, -)$ & $\sab(128,4)$ & $\fc(128, \relu, -)$ & \\
$\fc([128, 128, 128], \relu, -)$ & $\sab(128,4)$  & $\fc(256 \cdot 8, -, -)$ & \\
$\fc(128, \relu, -)$ & $\sab(128,4)$ & & \\
$\fc(128, -, -)$ & $\sab(128, 4)$ & & \\
\bottomrule
\end{tabular}
\label{tab:meta_anomaly_architecture}
\end{table}

Table~\ref{tab:meta_anomaly_architecture} describes the architecture for meta set anomaly experiments.
We trained all models via Adam optimizer with learning rate $10^{-4}$ and exponential decay of learning rate for 1,000 iterations.
1,000 datasets subsampled from CelebA dataset (see Figure~\ref{fig:celeba_figures}) are used to train and test all the methods.
We split 800 training datasets and 200 test datasets for the subsampled datasets.

\subsection{Point Cloud Classification}
\label{app:point_cloud}
We used the ModelNet40 dataset for our point cloud classification experiments.
This dataset consists of a three-dimensional representation of 9,843 training and 2,468 test data which each belong to one of 40 object classes.
As input to our architectures, we produce point clouds with $n=100, 1000, 5000$ points each (each point is represented by $(x,y,z)$ coordinates).
For generalization, we randomly rotate and scale each set during training.

We show results our architectures in Table~\ref{tab:pointcloud_architecture} and additional experiments which used $n=100, 5000$ points in Table~\ref{table:pointcloud}.
We trained using the Adam optimizer with an initial learning rate of $10^{-3}$ which we decayed by a factor of $0.3$ every 20,000 steps.
For the experiment with 5,000 points (Table~\ref{table:pointcloud}), we increased the dimension of the attention blocks ($\mathrm{ISAB}_{16}(512, 4)$ instead of $\mathrm{ISAB}_{16}(128, 4)$) and also decayed the weights by a factor of $10^{-7}$.
We also only used one ISAB block in the encoder because using two lead to overfitting in this setting.

\begin{table}[h]
\small
\centering
\caption{Detailed architectures used in the point cloud classification experiments.}
\vspace{5pt}
\begin{tabular}{@{}cccc@{}}\toprule
\multicolumn{2}{c}{Encoder} & \multicolumn{2}{c}{Decoder}\\
\cmidrule(lr){1-2} \cmidrule(lr){3-4}
rFF & ISAB & Pooling & PMA \\
\midrule
$\fc(256, \relu)$ & $\isab(256, 4)$ & $\max$& $\mathrm{Dropout}(0.5)$ \\
$\fc(256, \relu)$  & $\isab(256, 4)$ & $\mathrm{Dropout}(0.5)$ & $\pma_1(256, 4)$ \\
$\fc(256, \relu)$  & & $\fc(256, \relu)$ & $\mathrm{Dropout}(0.5)$ \\
$\fc(256, -)$ & &$\mathrm{Dropout}(0.5)$ & $\fc(40, -)$ \\
&&$\fc(40, -)$ &\\
\bottomrule
\end{tabular}
\label{tab:pointcloud_architecture}
\end{table}

\iffalse
\begin{table}[h]
\centering
\caption{Additional point cloud experiments using 100 points.}
\begin{tabular}{@{}ccccc@{}}
\toprule
Architecture & Accuracy\\
\midrule
rFF + Pooling & 0.7951 $\pm$ 0.0166 \\
rFF + PMA (1) & 0.8076 $\pm$ 0.0160 \\
ISAB (16) + Pooling & 0.8273 $\pm$ 0.0159 \\
Set Transformer (16) & \bf 0.8454 $\pm$ 0.0144 \\
\midrule
rFF + Pooling + tricks \citep{Zaheer2017} & 0.82 $\pm$ 0.02 \\
\bottomrule
\end{tabular}
\label{table:app_pointcloud100}
\vspace{10pt}
\centering
\caption{Additional point cloud experiments using 5000 points.}
\begin{tabular}{@{}ccccc@{}}
\toprule
Architecture & Accuracy\\
\midrule
rFF + Pooling & 0.8933 $\pm$ 0.0156 \\
rFF + PMA (1) & 0.8628 $\pm$ 0.0136 \\
ISAB (16) + Pooling & \bf 0.9040 $\pm$ 0.0173 \\
Set Transformer (16) & 0.8779 $\pm$ 0.0122 \\
\midrule
rFF + Pooling + tricks \citep{Zaheer2017} & \bf 0.90 $\pm$ 0.003 \\
\bottomrule
\end{tabular}
\label{table:app_pointcloud5k}
\end{table}
\fi

\section{Additional Experiments}
\subsection{Runtime of SAB and ISAB}
\begin{figure}[h]
\centering
\includegraphics[width=0.8\linewidth]{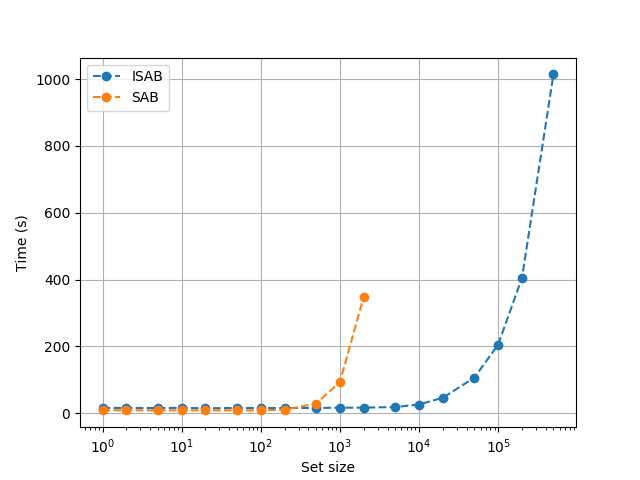}
\caption{
Runtime of a single SAB/ISAB block on dummy data.
x axis is the size of the input set and y axis is time (seconds).
Note that the x-axis is log-scale.
}
\label{fig:runtime}
\end{figure}
We measured the runtime of SAB and ISAB on a simple benchmark (Figure~\ref{fig:runtime}).
We used a single GPU (Tesla P40) for this experiment.
The input data was a constant (zero) tensor of $n$ three-dimensional vectors.
We report the number of seconds it took to process 10,000 sets of each size.
The maximum set size we report for SAB is 2,000 because the computation graph of bigger sets could not fit on our GPU.
The specific attention blocks used are $\mathrm{ISAB}_4(64, 8)$ and $\mathrm{SAB}(64, 8)$.

\bibliography{set_transformer}
\bibliographystyle{icml2019}